%% file: main.tex
\documentclass[runningheads]{llncs}

\usepackage{eccv}

\usepackage{eccvabbrv}

\usepackage{graphicx}
\usepackage{booktabs}

\usepackage{graphicx}%
\usepackage{multirow}%
\usepackage{amsmath,amssymb,amsfonts}%
\usepackage{mathrsfs}%
\usepackage[title]{appendix}%
\usepackage{xcolor}%
\usepackage{textcomp}%
\usepackage{manyfoot}%
\usepackage{booktabs}%
\usepackage{algorithm}%
\usepackage{algorithmicx}%
\usepackage{algpseudocode}%
\usepackage{listings}%
\usepackage{array}

\usepackage{bibunits}
\usepackage{ragged2e}

\newcommand{\NAME}{BaguanHR\xspace}
\newtheorem{thm}{Theorem}

\usepackage[accsupp]{axessibility}  

\usepackage{hyperref}

\usepackage{orcidlink}

\defaultbibliographystyle{splncs04}
\defaultbibliography{main}

\begin{document}

\title{Pushing the Limits of High-Resolution Weather Forecasting through Data Scaling} 

\titlerunning{High-Resolution Weather Forecasting via Data Scaling}

\author{
Yang Zhao\textsuperscript{*}\inst{1,2,4}\orcidlink{0009-0000-5556-1778} \and
Peisong Niu\textsuperscript{*}\inst{3,4}\orcidlink{0009-0007-7023-0900} \and
Tian Zhou\textsuperscript{*}\inst{3,4}\orcidlink{0000-0003-1789-5413}\and
Ziqing Ma\inst{4}\orcidlink{0000-0003-1567-5054}\and
Guanlong Ma\inst{1,2}\orcidlink{0009-0008-7350-3991}\and
Rong Jin\inst{4}\and
Huiling Yuan\textsuperscript{\dag}\inst{1,2}\orcidlink{0000-0003-4725-9039}\and
Liang Sun\textsuperscript{\dag}\inst{3,4}\orcidlink{0009-0002-5835-7259}
}

\authorrunning{Y. Zhao et al.}

\institute{
State Key Laboratory of Severe Weather Meteorological Science and Technology, Nanjing University, Nanjing, China \\
\and School of Atmospheric Science, Nanjing University, Nanjing, China \\
\and Ant Healthcare, Ant Group, Hangzhou, China \\
\and DAMO Academy, Alibaba Group, Hangzhou, China \\
\email{\{zhaoy2024,guanlongma\}@smail.nju.edu.cn}, \email{yuanhl@nju.edu.cn}
\email{\{niupeisong.nps,tian.zt,ls.537724\}@antgroup.com}
\email{maziqing.mzq@alibaba-inc.com},
\email{rongjinemail@gmail.com}
}

\maketitle

\begingroup
    \renewcommand{\thefootnote}{}
    \footnotetext{\textsuperscript{*}These authors contributed equally to this work.}
    \footnotetext{\textsuperscript{\dag}Corresponding authors.}
\endgroup

\begin{abstract}

\input{sections/0_abstract}
\end{abstract}

\section{Introduction}\label{sec1}
\input{sections/1_introduction_v2}
\section{Related Work}\label{sec2}

\input{sections/2_related_works}
\section{Rationale for Synthetic-to-Real Transfer}\label{sec3}
\input{sections/3_methods}
\section{Experiments}\label{sec4}
\input{sections/4_results}
\section{Conclusion}\label{sec4}
\input{sections/5_conclusion}
\section*{Acknowledgements}
\input{sections/6_acknowledge}

\newpage

%
%
\bibliographystyle{splncs04}
\bibliography{main}

\newpage
\appendix 
\section*{Supplementary Material} 
\input{sections/appendix}



\end{document}

%% file: sections/0_abstract.tex
The development of $0.1^{\circ}$ global weather forecasting models based on machine learning (ML) is constrained by the limited availability of high-resolution data, as decades of reanalysis are only available at $0.25^{\circ}$ resolution.
While existing approaches fine-tune $0.25^\circ$ forecast models on limited $0.1^\circ$ samples, we show that this transfer is hindered by the irreversible information loss inherent in coarse-resolution forecasting. 
Therefore, we propose \textbf{\NAME}, a framework that shifts the focus from transferring models to transferring data.
We first show that super-resolution (SR) has lower conditional entropy and input amplification than forecasting, making it a more robust vehicle for resolution transfer. By leveraging this advantage through variable-wise SR, 
we synthesize extensive $0.1^\circ$ data from ERA5. 
\NAME's performance on the synthetic-plus-real dataset exceeds both ML-based methods and IFS-HRES, achieving superior performance across over 85\% of the lead times within 72 hours.
Furthermore, our findings highlight a power-law scaling effect, as a twofold increase in data reduces RMSE by 4.6\% for 72-hour forecasting and 4.9\% for 120-hour forecasting.
Our results demonstrate that scaling high-resolution ML-based forecasting is primarily a data bottleneck, and that variable-wise super-resolution provides a simple yet general solution to unlock long coarse-resolution reanalyses for high-resolution training.

%% file: sections/1_introduction_v2.tex
The field of global weather forecasting is witnessing a profound change, with data-driven models now outperforming established numerical weather prediction systems at medium ranges~\cite{pangu_nature, niu2025utilizing, graphcast, Fuxi_nature, Chen2023FengWuPT,lang2024aifsecmwfsdatadriven,pathak2022fourcastnet}.
However, advancing these models toward higher spatial resolution remains fundamentally constrained by the availability of long, consistent, high-resolution global datasets. For example, the European Center for Medium-Range Weather Forecasts (ECMWF) operational analysis provides global $0.1^\circ$ (roughly 9km) data only after 2016~\cite{malardel2016new}, yielding at most about 10 years of training data up to 2026, which is far too limited for training large-scale ML forecast models without severe overfitting. In contrast, 
reanalysis data such as ERA5~\cite{hersbach2020era5} provides 47 years of data (up to 2026) only at a coarser $0.25^\circ$ resolution,
creating a resolution--data imbalance that has become a critical bottleneck for high-resolution ML weather forecasting.

A recent line of work attempts to overcome this challenge by transferring a pretrained $0.25^\circ$ forecast model to $0.1^\circ$ through architectural design and fine-tuning on the small amount of $0.1^{\circ}$ high-resolution data~\cite{Bodnar2024AuroraAF, han2024fengwu}. This strategy leverages prior dynamical knowledge encoded in the coarse model but is inherently limited by two factors. First, the pretrained coarse-resolution model cannot fully exploit the rich information contained in the $0.1^\circ$ analysis fields because its learned representation is tied to the low-resolution grid. Second, weather prediction is a significantly more complex task than downscaling; adapting a coarse-resolution forecast model to a finer grid requires learning both fine-scale processes and the correct multiscale dynamics, which is difficult to achieve through limited fine-tuning. As a result, coarse-to-fine transfer often yields restricted accuracy improvements and depends heavily on sophisticated architectural modifications.

We propose a paradigm shift: transfer data, not models. \textbf{Our core insight is that super-resolution (SR) is a same-time conditional reconstruction task with provably lower conditional entropy than multi-step forecasting.} Following this insight, SR should be (i) intrinsically easier to learn, (ii) more stable under imperfect inputs, and (iii) suitable as a scalable mechanism to break the data bottleneck in high-resolution forecasting. We test these implications through a set of experiments that sequentially close this reasoning loop.

\begin{figure}[htbp]
    \centering
    \includegraphics[width=0.9\textwidth]{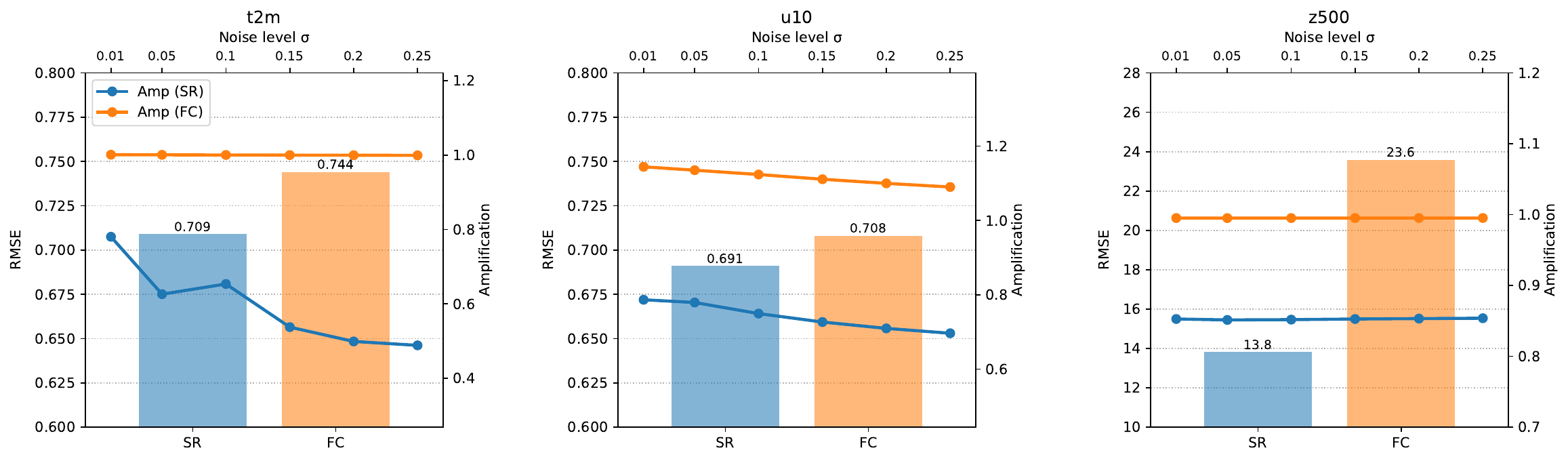}
    \caption{Entropy-based Robustness evaluation for t2m, u10, and z500. The $x$-axis corresponds to different noise levels, and the $y$-axis is the RMSE and amplification factor. \textbf{Lower amplification factors reflect greater inherent stability.}}
    \label{fig:why_data_transfer_works}
\end{figure}

We investigate the mechanisms underlying our approach through two primary lenses. First, we analyze the robustness of SR and forecasting (FC) for t2m, u10 and z500, as shown in Fig.~\ref{fig:why_data_transfer_works}. By measuring the error growth per unit of injected Gaussian noise (amplification factor) for both tasks, we find that: (i) \textbf{SR is a simpler task}, yielding much lower RMSE than one-step forecasting (e.g., 13.8 vs.\ 23.6 for z500), which aligns with \cite{mardani2023generative, guen2020disentangling}; and (ii) \textbf{SR is more stable}, with an amplification factor ($0.6$--$0.8\times$) significantly lower than that of forecasting ($1.0$--$1.2\times$).
Second, our \textbf{data scaling law analysis} (Sec.~\ref{results:scaling_of_data}) demonstrates that $0.1^\circ$ forecasting is strongly data-limited, exhibiting power-law improvements as training volume increases. Specifically, expanding the training data from 7 to 18 years reduces RMSE by more than 4.5\% at long lead times.

Building on these findings, we employ a \textbf{synthetic-plus-real} strategy. Specifically, we train per-variable SR models on limited paired $0.1^\circ$--$0.25^\circ$ samples, apply them to decades of ERA5 to synthesize high-resolution $0.1^\circ$ training data, and then train $0.1^\circ$ forecast models from scratch. This data-transfer approach substantially outperforms coarse-to-fine baselines without specialized architectures or fine-tuning, indicating that scaling high-resolution ML forecasting is best addressed by data construction via variable-wise super-resolution. 

The main contributions of our work are as follows:
\begin{itemize}

\item We propose \NAME, a framework leveraging the low conditional entropy of super-resolution (SR) as a robust data augmentation strategy for high-resolution weather forecasting. \NAME outperforms various ML-based models and IFS-HRES. Specifically, over \textbf{85\%} of lead times show superior performance within 72 hours, highlighted by a \textbf{4.0\%} RMSE reduction compared specifically to IFS-HRES.


\item We conduct a comprehensive investigation from both experimental and theoretical perspectives to elucidate the mechanisms underlying the effectiveness of synthetic-plus-real strategy. By leveraging the inherent robustness of per-variable SR, our approach achieves up to a 20\% reduction in RMSE compared to traditional coarse-to-fine methods, establishing a more stable framework for high-resolution modeling.

\item We establish fundamental data scaling laws in high-resolution forecasting, revealing that performance improves following a power law as data volume increases. Specifically, expanding the training dataset from 7 to 18 years yields a \textbf{4.9\%} reduction in RMSE for long-term (120-hour) forecasts and \textbf{4.6\%} for 72-hour forecasts.

\item Our study highlights a shift in perspective: for high-resolution ML-based weather forecasting, data construction via super-resolution proves more scalable and effective than architectural design or transfer learning. As the data availability is the primary bottleneck in high-resolution forecasting, our paradigm addresses the data directly, and can be expanded to other high-resolution tasks. 


\end{itemize}

%% file: sections/2_related_works.tex
\subsection{Data-Driven Weather Forecasting Models}

Moving beyond the resolution limits of $0.25^{\circ}$ models~\cite{graphcast, pangu_nature, Fuxi_nature, Chen2023FengWuPT, niu2025utilizing, Bodnar2024AuroraAF}, recent $0.1^{\circ}$ forecasting efforts~\cite{Bodnar2024AuroraAF, han2024fengwu} focus on transferring knowledge from pretrained coarse-grained systems. Aurora~\cite{Bodnar2024AuroraAF} introduces bilinear upscaling for patch alignment, while Fengwu-GHR~\cite{han2024fengwu} employs identity-preserving mappings for high-resolution extrapolation. These methods highlight a growing trend of leveraging transfer learning to achieve operational-grade precision, but they still treat high-resolution forecasting primarily as a model-transfer problem under limited $0.1^{\circ}$ data. In contrast, our approach tackles the same challenge from a data-centric perspective by expanding the effective $0.1^{\circ}$ dataset via super-resolution. This data-centric route is complementary to neural-operator methods, including Fourier Neural Operators and coarse-graining operators~\cite{DBLP:conf/iclr/LiKALBSA21, wang2024coarse}, which address multi-resolution physical modeling from an architecture-centric perspective.

\subsection{Scaling in Weather Forecasting Task}


Inspired by the scaling laws in NLP~\cite{kaplan2020scaling} and CV~\cite{zhai2022scaling, cherti2023reproducible}, recent studies have begun exploring the relationship between compute, data, and performance in meteorology. While~\cite{niu2025utilizing} analyzed model and data scaling at resolutions below $0.25^{\circ}$, \cite{yu2025scaling} reported fundamental power laws across model size and compute budget. However, \textbf{whether these performance gains persist when scaling data at $0.1^{\circ}$ resolution} remains an open question, especially given the scarcity of high-resolution analysis data. This motivates our investigation into effective data-scaling strategies within high-resolution regimes.

%% file: sections/3_methods.tex
\subsection{Robustness Evaluation: Super-Resolution vs. Forecasting}

To justify the adoption of Super-Resolution as a superior data augmentation strategy, we conduct a comparative analysis between two paradigms: forecasting (temporal evolution) and super-resolution (spatial reconstruction) under varying levels of input noise. The formulations for these tasks are defined as follows:

\begin{itemize}
    \item \textbf{High-Resolution Forecasting (FC):} \NAME is designed to generate global medium-range forecasts at a $0.1^{\circ}$ spatial resolution ($H=1801, W=3600$). Let $X_t \in \mathbb{R}^{V \times H \times W}$ represent the atmospheric state at time $t$, where $V$ denotes the number of variables. The forecasting model aims to learn a mapping $\Phi_F: X_t \to X_{t+\Delta t}$, where $\Delta t=6$ hours. Future states are predicted in an auto-regressive manner, utilizing the previous output as input for the subsequent step.
    
    \item \textbf{Super-Resolution (SR):} To augment the training data, we define a spatial reconstruction task to bridge the gap between coarse reanalysis and high-resolution analysis data. Given a low-resolution state $Z_t \in \mathbb{R}^{V \times H' \times W'}$ ($0.25^{\circ}$), the SR model learns a mapping $\Phi_{SR}: Z_t \to \hat{X}_t$, producing a $0.1^{\circ}$ pseudo-label $\hat{X}_t$ that maintains physical consistency with the source data while enhancing spatial detail.
\end{itemize}

To evaluate model robustness, we perturb input fields with Gaussian noise $\delta \sim \mathcal{N}(0, \sigma^2)$ at varying intensities $\sigma$. Sensitivity is quantified by an \textbf{amplification factor}, defined as the change in prediction error relative to the baseline, normalized by the noise magnitude: $(|\text{RMSE}_\sigma - \text{RMSE}_0|/\sigma)$. This metric captures the marginal response of SR and FC models to input perturbations.
By isolating the incremental error caused solely by input perturbations, this factor reveals the intrinsic stability of the model's mapping function. A high amplification factor suggests that small input fluctuations are significantly magnified during the inference process, whereas a low factor signifies a resilient architecture.

Our analysis reveals that the SR approach offers distinct advantages over the FC method in two key aspects. First, as illustrated in Fig.~\ref{fig:why_data_transfer_works}, SR consistently exhibits significantly lower amplification factors than forecasting across all tested meteorological variables, including t2m, u10, and z500. Second, SR maintains a substantial advantage in RMSE compared to the FC model at every noise intensity. Together, these findings suggest that SR is fundamentally less sensitive to input uncertainties and more effective at suppressing error propagation. These results demonstrate that \textbf{SR’s inherent stability and resilience to input perturbations validate its application in robust data augmentation}.

\subsection{Synthetic-plus-Real vs. Coarse-to-Fine Adaptation}

\begin{figure}[htbp]
    \centering
    \includegraphics[width=0.85\textwidth]{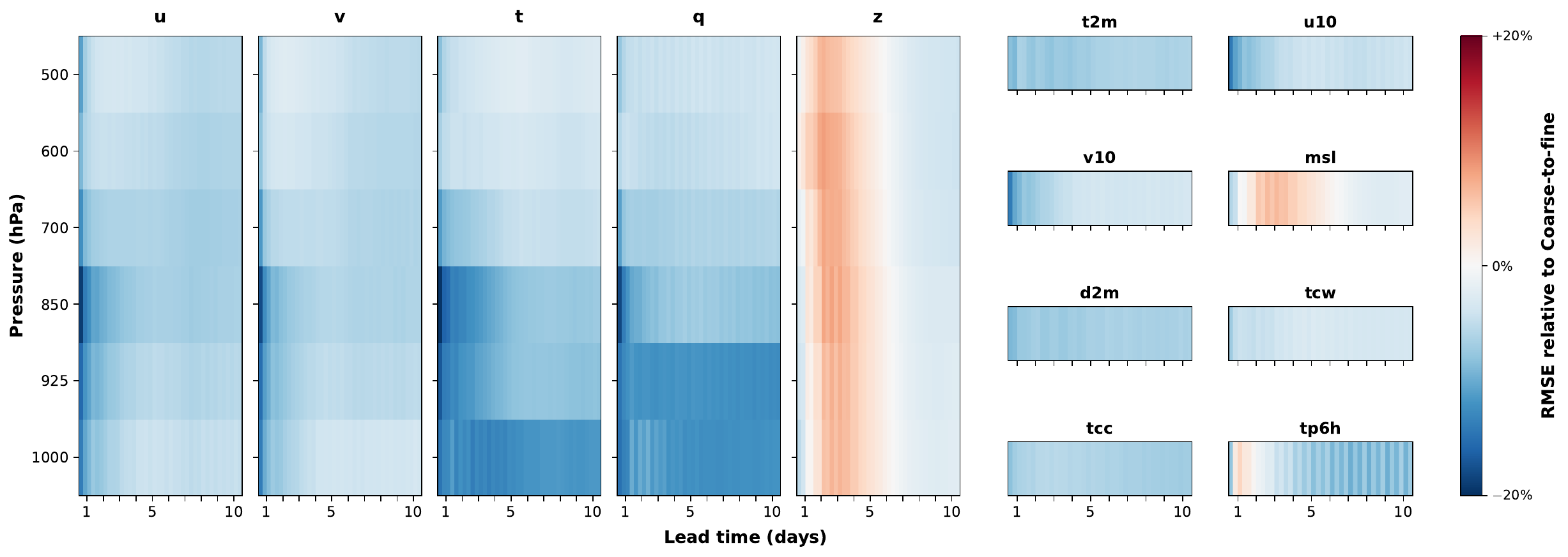}
    \caption{Comparative performance of the ``synthetic-plus-real'' training-from-scratch paradigm versus the ``coarse-to-fine transfer'' approach. The $x$-axis is  the lead time (days), and different colors show different values of relative RMSE difference.}
    \label{fig:comparision_to_coarse_to_fine}
\end{figure}

At the onset of our study, we conduct a comprehensive comparison between two technical paradigms for high-resolution adaptation: the conventional ``coarse-to-fine'' approach and our proposed ``synthetic-plus-real'' strategy. This comparison serves to identify \textbf{whether model-level transfer or data-driven augmentation provides a more effective route for $0.1^\circ$ modeling}.

\subsubsection{Experimental Setup}
Current $0.1^\circ$ resolution models, represented primarily by \cite{han2024fengwu} and \cite{Bodnar2024AuroraAF}, both follow a `coarse-to-fine' technical route. In this work, we choose to implement our scheme by reproducing the methodology established by \cite{han2024fengwu}, which adapts a pre-trained $0.25^\circ$ model to $0.1^\circ$ resolution through decompositional and combinational transfer learning.
In contrast, the ``synthetic-plus-real'' paradigm, as described in Sec.~\ref{section:method}, augments the training set by generating high-resolution synthetic data and trains the $0.1^\circ$ model from scratch. To ensure a fair comparison, both paradigms are evaluated on 2025 real-time analysis data.

\subsubsection{Training Datasets}
\label{sec:dataset}
The ``synthetic-plus-real'' training set integrates two key components:
\begin{itemize}
\item \textbf{Real-Time Analysis Data:} We utilize $0.1^\circ$ ECMWF 4D-Var analysis data (2017--2024). For accumulated variables (fdir, ssrd, and tp) unavailable in the analysis archive, we source them from zero-hour IFS-HRES forecasts.
\item \textbf{Synthetic Data:} To supplement the limited $0.1^\circ$ archive, we apply a per-variable super-resolution (SR) model to $0.25^\circ$ ERA5 reanalysis (2007--2016). This yields 10 years of $0.1^\circ$ pseudo-labels that maintain the physical consistency of the original ERA5 records.
\end{itemize}
The complete dataset encompasses 13 standard pressure levels (50--1000 hPa) for upper-air variables (z, t, u, v, q) and a comprehensive set of surface parameters, including standard meteorological fields (t2m, d2m, u10, v10, msl, sp), specialized variables (u100, v100, lcc, tcc, sst), radiation metrics (ssrd, fdir), and moisture metrics (tcw, tcwv, tp). This dataset serves as the foundation for the experimental results presented in Sec.~\ref{section:exp}.

\subsubsection{Results}
As illustrated in Fig.~\ref{fig:comparision_to_coarse_to_fine}, the synthetic-plus-real paradigm significantly outperforms the coarse-to-fine baseline across the majority of variables. Notably, RMSE reductions reach as much as 20\% (highlighted by the prominent blue regions). These findings provide empirical evidence for \textbf{why data transfer works}: training directly on high-fidelity synthetic data facilitates the learning of more robust spatial representations than parameter-level fine-tuning on interpolated inputs. 
Thus, \textbf{we adopt the synthetic-plus-real paradigm as our core methodology for all subsequent stages}.

\subsection{Theoretical Justification of Synthetic Data Transfer}
\input{sections/3.1_theory}

\section{System Architecture and Training Strategies}
\label{section:method}
\subsection{Framework of \NAME}

As illustrated in Fig.~\ref{fig:arch}, \NAME comprises a per-variable super-resolution generator and a high-resolution forecasting model. To mitigate the scarcity of $0.1^{\circ}$ EC analysis data, we employ the SR generator to synthesize $0.1^{\circ}$ pseudo-labels from $0.25^{\circ}$ ERA5 data. 

\begin{figure}[h]
    \centering
    \includegraphics[width=0.9\textwidth]{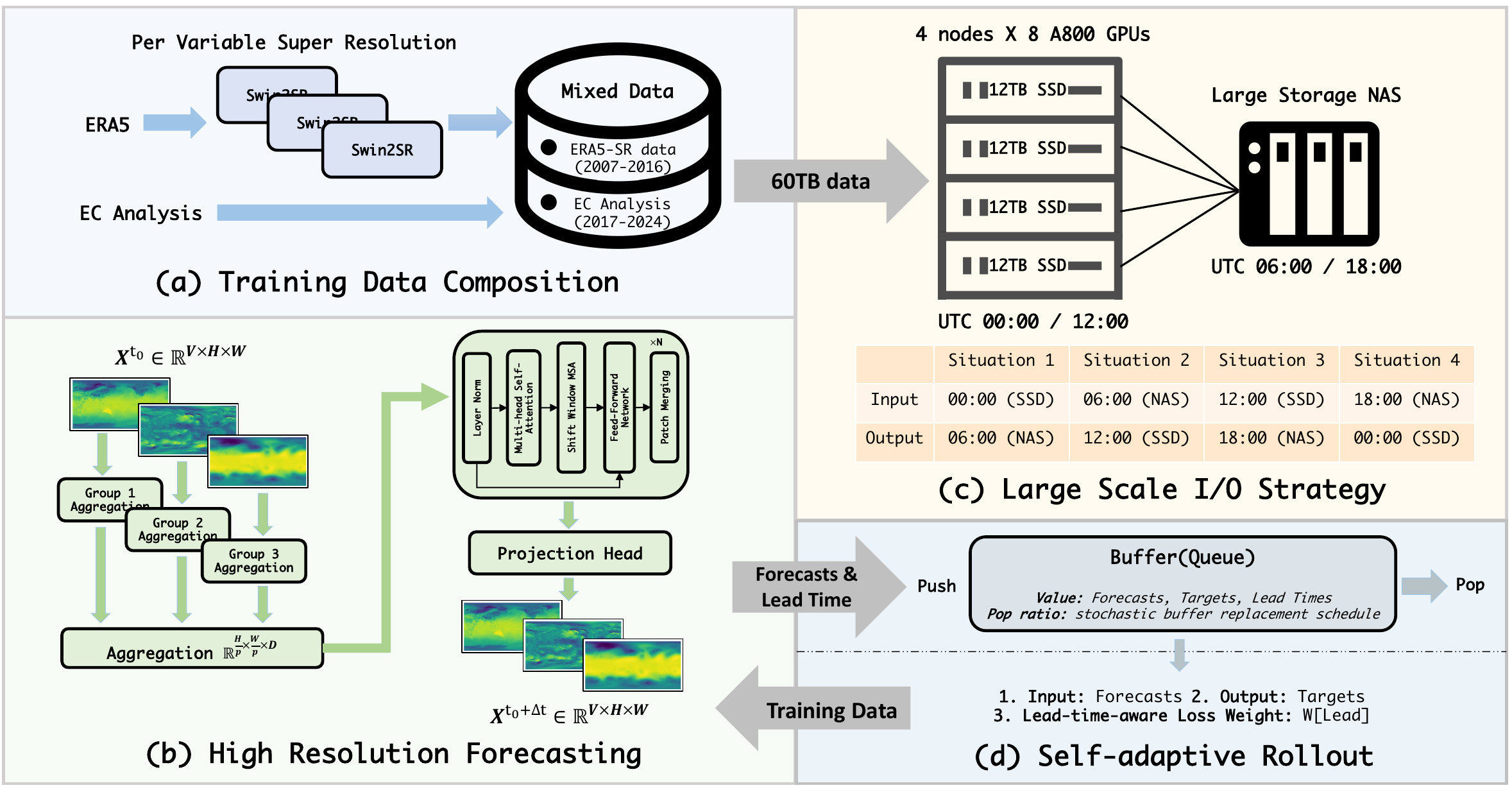}
    \caption{Overview of \NAME's architecture. (a) Training Data Composition. (b) High-resolution forecasting model. (c) Large-scale I/O strategy. (d) Self-adaptive rollout strategy.}
    \label{fig:arch}
\end{figure}

\textbf{Per-Variable Super-Resolution}
Unlike weather forecasting, which is formulated as a physics-constrained initial value problem, super-resolution emphasizes topographic-guided spatial reconstruction. Given that variable-specific training yields higher fidelity~\cite{vandal2017deepsd, harris2022generative, addison2022machine}, we adopt Swin2SR~\cite{conde2022swin2sr} to downscale $0.25^{\circ}$ ERA5 data, treating each atmospheric variable independently. Each SR model is trained on paired 0.25$^\circ$--0.1$^\circ$ fields using an MAE reconstruction objective.

\textbf{High-Resolution Forecasting Model}
As illustrated in Fig.~\ref{fig:arch}(b), the forecasting model processes inputs through variable-specific tokenization and hierarchical weather embedding, followed by multiple Swin Transformer blocks~\cite{liu2021Swin}, and a linear projection head for spatial restoration.



\paragraph{Variable-Specific Tokenization and Hierarchical Embedding.}
Given an input of shape $V \times H \times W \times D$, where $V$ is 80 (tp solely as an output), $H$ is 1801 and $W$ is 3600. Then each variable is independently tokenized via patchification with a patch size $p$, resulting in a reshaped tensor of $V \times h \times w \times D$ (where $h = H/p$ and $w = W/p$). However, the ``brute-force'' cross-attention aggregation common in prior works~\cite{Nguyen2023ClimaXAF, stomer, niu2025utilizing} becomes computationally prohibitive for such high-resolution atmospheric data. To mitigate this, we propose a \textbf{hierarchical weather embedding}. By partitioning variables into 12 groups, our model first distills them into group-specific queries and then fuses them into a global summary. This two-level aggregation strategy effectively minimizes GPU memory overhead while streamlining the training process.

\paragraph{Swin Transformer Blocks and Projection Head.}
The embeddings are processed by Swin Transformer blocks to facilitate patch interactions. Finally, a linear projection head ($\mathbb{R}^{D\times Vp^2}$) restores the representation to its original spatial dimensions.

\subsection{Self-Adaptive Rollout Strategy}

To mitigate error accumulation in auto-regressive forecasting, recent models~\cite{han2024fengwu, Bodnar2024AuroraAF} have adopted RL-inspired replay buffers for long-range stability. However, over-emphasizing long-lead horizons often degrades short-term accuracy due to a forgetting effect. To resolve this, we introduce two strategies for optimized replay buffer utilization that balance long-range stability with short-term fidelity.

\begin{itemize}
    \item \textbf{Lead-Time-Aware Loss Weighting:} We employ lead-time-dependent wei\-ghts to ensure short-lead predictions retain sufficient gradient influence. This prevents long-lead cumulative errors from dominating the optimization, thereby preserving foundational accuracy and mitigating error propagation during auto-regression.

    \item \textbf{Stochastic Buffer Replacement} To enhance data diversity, we implement a randomized buffer update strategy rather than a standard FIFO approach. By maintaining a heterogeneous mixture of lead times per batch, this method stabilizes distributed training and mitigates the instability typically caused by small local batch sizes.
\end{itemize}

\subsection{Training and I/O Optimization}

We also implement a scalable training framework for \NAME, leveraging advanced I/O partitioning and multi-stage optimization to ensure high-fidelity forecasting and high-throughput training.

\begin{itemize}
    \item \textbf{Training Pipeline} \NAME utilizes 32 NVIDIA A800 GPUs in three stages: (1) \textbf{Pretraining}, for 250k steps (14 days) on 8-year EC analysis data; (2) \textbf{Synthetic Fine-tuning}, for 70k steps (4 days) incorporating 10-year super-resolution data; and (3) \textbf{Rollout Training}, for 2 days with a self-adaptive replay buffer to suppress long-range error accumulation.

    \item \textbf{I/O Strategies} To address the I/O bottleneck caused by 32-bit data files of 4--5 GB each, we \textit{achieve a $4\times$ speedup} via: (1) \textbf{Data Partitioning}, sharding data across local 12 TB SSDs per node; (2) \textbf{Hybrid Storage}, caching UTC 00:00/12:00 data on SSDs while keeping 06:00/18:00 data on NAS.
\end{itemize}

\noindent Overall, the training pipeline takes about 20 days on 32 NVIDIA A800 GPUs and uses about 60 TB of storage for 18 years of 0.1$^\circ$ synthetic-plus-real data;
detailed compute and I/O costs are provided in Supplementary Sec.~C.4.

%% file: sections/3.1_theory.tex
\def \E {\mathrm{E}}
\def \x {\mathbf{x}}
\def \g {\mathbf{g}}
\def \L {\mathcal{L}}
\def \D {\mathcal{D}}
\def \z {\mathbf{z}}
\def \u {\mathbf{u}}
\def \H {\mathcal{H}}
\def \w {\mathbf{w}}
\def \R {\mathbb{R}}
\def \S {\mathcal{S}}
\def \regret {\mbox{regret}}
\def \Uh {\widehat{U}}
\def \Q {\mathcal{Q}}
\def \W {\mathcal{W}}
\def \N {\mathcal{N}}
\def \A {\mathcal{A}}
\def \q {\mathbf{q}}
\def \v {\mathbf{v}}
\def \M {\mathcal{M}}
\def \c {\mathbf{c}}
\def \ph {\widehat{p}}
\def \d {\mathbf{d}}
\def \p {\mathbf{p}}
\def \q {\mathbf{q}}
\def \db {\bar{\d}}
\def \dbb {\bar{d}}
\def \I {\mathcal{I}}
\def \f {\mathbf{f}}
\def \a {\mathbf{a}}
\def \b {\mathbf{b}}
\def \ft {\widetilde{\f}}
\def \bt {\widetilde{\b}}
\def \h {\mathbf{h}}
\def \B {\mathbf{B}}
\def \bts {\widetilde{b}}
\def \fts {\widetilde{f}}
\def \Gh {\widehat{G}}
\def \bh {\widehat{b}}
\def \fh {\widehat{f}}
\def \vb {\bar{v}}
\def \zt {\widetilde{\z}}
\def \zts {\widetilde{z}}
\def \s {\mathbf{s}}
\def \gh {\widehat{\g}}
\def \vh {\widehat{\v}}
\def \Sh {\widehat{S}}
\def \rhoh {\widehat{\rho}}
\def \hh {\widehat{\h}}
\def \C {\mathcal{C}}
\def \V {\mathcal{V}}
\def \t {\mathbf{t}}
\def \xh {\widehat{\x}}
\def \Ut {\widetilde{U}}
\def \wt {\widetilde{w}}
\def \Th {\widehat{T}}
\def \Ot {\tilde{\mathcal{O}}}
\def \X {\mathcal{X}}
\def \nb {\widehat{\nabla}}
\def \K {\mathcal{K}}
\def \P {\mathbb{P}}
\def \T {\mathcal{T}}
\def \F {\mathcal{F}}
\def \ft{\widetilde{f}}
\def \xt {\widetilde{x}}
\def \Rt {\mathcal{R}}
\def \V {\mathcal{V}}
\def \Rb {\bar{\Rt}}
\def \wb {\bar{\w}}
\def \fh {\widehat{f}}
\def \wh {\widehat{\w}}
\def \lh {\widehat{\lambda}}
\def \e {\mathbf{e}}
\def \B {\mathcal{B}}
\def \P {\mathcal{P}}
\def \vb {\bar{v}}
\def \ub {\bar{u}}
\def \Rt {\mathcal{R}}
\def \wh {\widehat{w}}
\def \Lh {\widehat{\L}}
\def \rh {\widehat{r}}
\def \G {\mathcal{G}}
\def \qh {\widehat{q}}
\def \Qh {\widehat{Q}}
\def \Gh {\widehat{G}}
\def \Ph {\widehat{P}}
\def \U {\mathcal{U}}
\def \diag {\text{diag}}
\def \Kh {\widehat{K}}
\def \ut {\tilde{u}}
\def \gb {\bar{\gamma}}
\def \gh {\widehat{g}}
\def \Vh {\widehat{V}}
\def \zh {\widehat{z}}
\def \zb {\bar{z}}
\def \vb {\bar{v}}
\def \pt {\widetilde{p}}
\def \Hh {\widehat{H}}
\def \sgn {\mbox{sgn}}
\def \eh {\hat{\epsilon}}
\def \Sh {\widehat{\Sigma}}
\def \St {\widetilde{\Sigma}}
\def \Vt {\widetilde{V}}
\def \tr {\mbox{tr}}
\def \yh {\widehat{y}}
\def \yt {\widetilde{y}}
\def \rh {\widehat{r}}
\def \Lt {\widetilde{\L}}
\def \Zt {\widetilde{Z}}
\def \muh {\widehat{\mu}}

\renewcommand{\algorithmicrequire}{ \textbf{Input:}} 
\renewcommand{\algorithmicensure}{ \textbf{Output:}} 

While including clean data can mitigate \textit{model collapse} in iterative learning \cite{shumailov-2024-model-collapse, dey-2024-universality}, the consensus remains that synthetic data generally impairs generalization \cite{amin-2026-limitation-of-ERM}. We challenge this view by demonstrating that an appropriate synthetic generation procedure can significantly reduce generalization error in linear regression. 


Consider a multi-cluster distribution where samples $(x, y)$ are generated via $x = u \otimes e_k$ with $u \sim \mathcal{N}(\mu_k, \sigma^2 I)$ and $y = \langle w, x \rangle + z$ with noise $z \sim \mathcal{N}(0, \gamma^2)$. Let $D_a = \{(x_i, y_i)\}_{i=1}^n$ denote the training set. Under assumptions on $\mu_k$ (normalization, orthogonality, and separation) and $\sigma = o(1)$, we obtain an estimation of $w$(denoted as $\wh$) using the standard linear regression model: 
\begin{eqnarray}
\wh = \left(\frac{1}{n}\sum_{i=1}^n x_ix_i^{\top}\right)^{-1}\left(\frac{1}{n}\sum_{i=1}^n x_i y_i\right).  \label{eqn:wh}
\end{eqnarray}
The following theorem shows the estimation error for $\wh$ in Eq.~\eqref{eqn:wh}.

\begin{thm}
With a probability $\geq$ $1 - \delta$, we have
\begin{eqnarray}
\left|w - \wh\right| \leq O\left(\gamma\sqrt{\frac{dm}{n}\log\frac{1}{\delta}}\right).
\end{eqnarray}
\end{thm}

Suppose that we have $N$ additional noisy examples $D_b = \{(u_i, y_i)\}_{i=1}^N$ ($N \gg n$), where each $u_i$ is generated via the same process but lacks its cluster index $k_i$. Because this key information is latent, we refer to $D_b$ as noisy data. The following theorem indicates that a linear regression model learned directly from $D_b$ will suffer a substantially larger error compared to $\hat{w}$ defined in Eq.~\eqref{eqn:wh}.
\begin{thm}
Let $w_u \in \R^d$ be the linear regression model we learned from $\D_b$. Assume $N = +\infty$. We have
\begin{eqnarray}
\E_{k \in [m]}\E_{u \sim \N(\mu_k, \sigma^2 I_d)} \left[\left|\langle u, w_u\rangle - \langle u, w_k\rangle\right|^2\right] \geq \sigma^2\mbox{VAR}\left(w_1, \ldots, w_m\right),
\end{eqnarray}
where
\begin{eqnarray}
\mbox{VAR}\left(w_1, \ldots, w_m\right) = \E_{k \in [m]}\left[\left|w_k - \E_j\left[w_j\right]\right|^2\right].
\end{eqnarray}
\end{thm}

Compared to $\wh$ in Eq.~\eqref{eqn:wh}, we will have $\wh$ yields a smaller regression error than $w_u$, even when $N \rightarrow +\infty$.
Our second approach is to learn from $\D_a$ a model to predict the index of Gaussian distribution, and apply the learned prediction model to augment the input patterns with the predicted index.
The following theorem bounds the recovery error for $\wt$.

\begin{thm}
Assume $n \geq md\log\frac{N}{\delta}$.
Then, with a probability at least $1 - \delta$, we have
\begin{eqnarray}
\left|\wt - w\right| \leq O\left(\sqrt{\frac{dm}{N}\log\frac{1}{\delta}}\right).
\end{eqnarray}
\end{thm}
Comparing Theorem 3 to Theorem 1, we can see that the solution estimated from the synthetical dataset $\D_c$ is significantly closer to the oracle $w$ than the one that is learned from the clean dataset $\D_a$. \textbf{This result theoretically validates the feasibility of employing super-resolution models for robust data augmentation}.
Further details and proofs are provided in the supplementary material.

%% file: sections/4_results.tex




\label{section:exp}
\subsection{Dataset \& Metrics}
Detailed data descriptions are provided in Sec.~\ref{sec:dataset}. 
We adopt the standard weather evaluation methodology outlined in WeatherBench~\cite{rasp2024weatherbench2benchmarkgeneration}, focusing on forecasts initialized at 00/12 UTC for the year 2025. 
We use weighted Root Mean Square Error (RMSE) and weighted Anomaly Correlation Coefficient (ACC), which have also been widely used in AI-based weather models~\cite{niu2025utilizing, graphcast, Fuxi_nature, pangu_nature}, as our main metrics.



\subsection{The Scaling of Data Volume}
\label{results:scaling_of_data}

Fig.~\ref{fig:scaling_law} reveals consistent power-law-like improvements in forecast skill as the volume of training data increases. This confirms the pivotal role of data scaling in ML-based weather models. Beyond this overall trend, we observe:

\begin{figure}[htbp]
    \centering
    \includegraphics[width=0.9\textwidth]{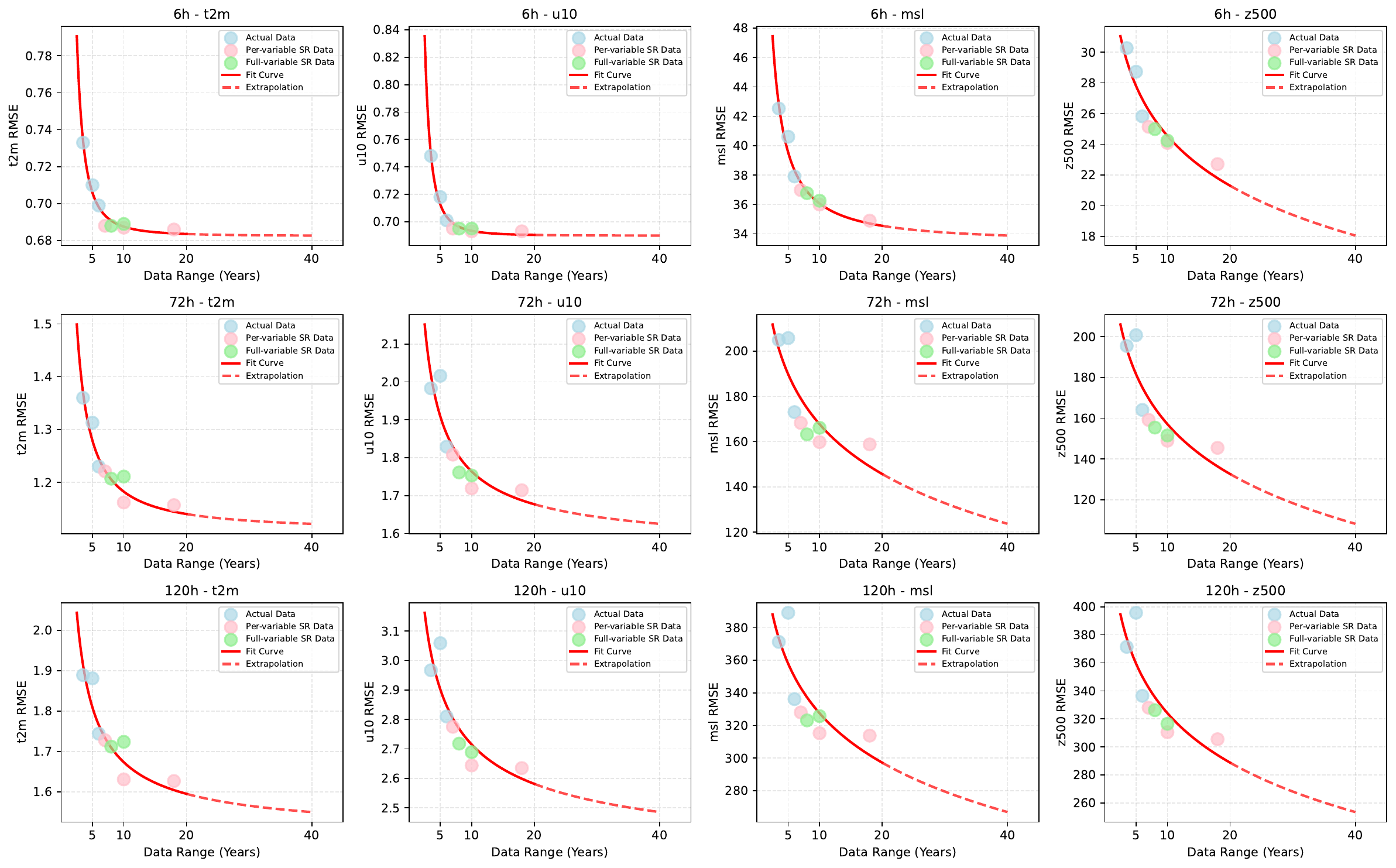}
    \caption{Scaling law curves for different variables and lead times. 
    Different subplots correspond to different variables and forecasting horizons. We plot the $x$-axis in log scale to reflect the scaling law.}
    \label{fig:scaling_law}
\end{figure}

\begin{enumerate}
    \item \textbf{Scaling Consistency Across Lead Times:} 
    The power-law behavior of error reduction remains remarkably stable across different forecast lead times. This consistent scaling slope suggests that increasing data volume is a universally robust strategy for improving both short-term deterministic and long-term atmospheric predictions.
    Notably, for long-term forecasting, RMSE is reduced by \textbf{4.6\%} at 72 hours and \textbf{4.9\%} at 120 hours.

    \item \textbf{Theoretical Saturation and Diminishing Returns:} To provide practical guidance for large-scale training, we conducted an
extrapolation analysis beyond the 18-year data threshold. While the full
47-year ERA5 archive could be leveraged for data augmentation via SR, this
extrapolation should be interpreted as a trend analysis rather than a definitive
performance estimate. The fitted curves suggest that marginal performance gains
may gradually diminish beyond 18 years, indicating a practical trade-off where
further data expansion yields sub-linear improvements relative to the increasing
storage and I/O cost.

    \item \textbf{Variable-Specific Sensitivity to Data Volume:} 
    Sensitivity to data scaling varies significantly by physical field: while most variables (e.g., z500, t2m) plateau early, others like msl exhibit sustained headroom for improvement. This heterogeneity informs our decision to standardize on an 18-year period to optimize overall storage and I/O efficiency.
    
    \item \textbf{Efficacy of Granular Data Augmentation:} 
    Comparative analysis shows that models trained on data generated by per-variable SR outperform those trained on full-variable SR output at a 10-year scale. This advantage is particularly pronounced for surface variables (e.g., t2m), highlighting the value of decoupled strategies in data synthesis.

\end{enumerate}

\subsection{Performance Evaluation and Benchmarking}

To evaluate the model's efficacy, we adopt a dual-perspective approach: first, a holistic global assessment across multiple variables to establish baseline predictive skill, followed by a rigorous investigation into model robustness under diverse extreme scenarios.

\subsubsection{General Comparison of Variables}

To show the superiority of \NAME, we evaluate it against both
physics-based and ML-based baselines. IFS-HRES~\cite{ecmwf_hres} serves as the
operational physics-based 0.1$^\circ$ reference. For ML baselines, we use
Baguan+swin2sr as the primary comparison, where 0.25$^\circ$ Baguan forecasts
are adapted to the 0.1$^\circ$ grid using Swin2SR. Since the first 3-day
forecast is more accurate, directly actionable, and therefore especially
important, we further include GraphCast+swin2sr, Pangu+swin2sr, and Baguan+GHR
for short-term lead times. As Fengwu-GHR~\cite{han2024fengwu} cannot be
directly replicated, Baguan+GHR is re-implemented based on
Baguan~\cite{niu2025utilizing} as a coarse-to-fine high-resolution adaptation
baseline.

\begin{figure}[h]
    \centering
    \includegraphics[width=0.9\textwidth]{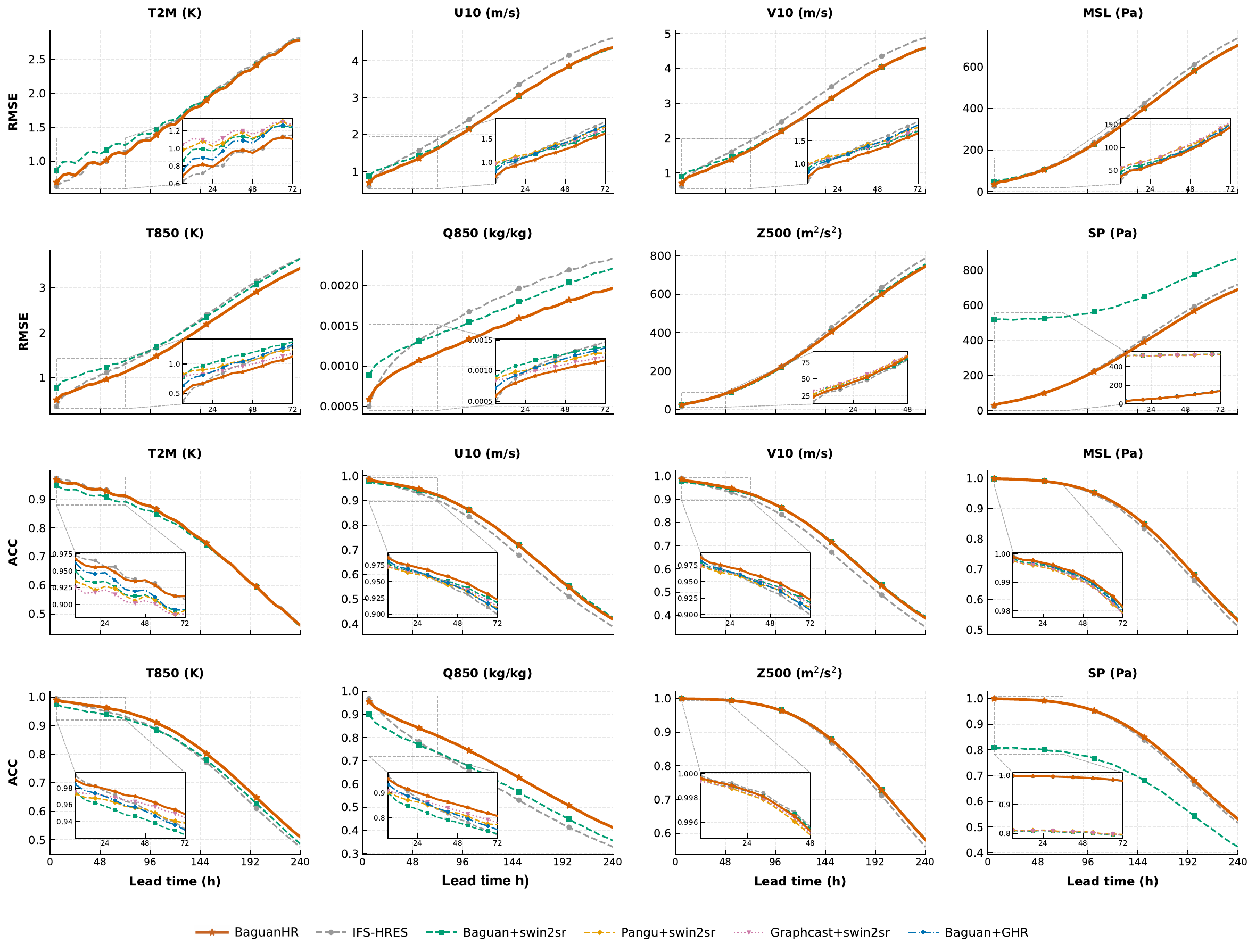}
    \caption{Overall performance of \NAME against several baselines. 
    `Baguan+GHR' denotes the version replicated from~\cite{han2024fengwu}. 
    We zoom in the first 72 hours for better comparison among different methods. }
    \label{fig:main_table}
\end{figure}

\NAME surpasses the baselines (Fig.~\ref{fig:main_table}), improving over \textbf{85\%} of lead times within 72 hours and reducing RMSE by \textbf{4.0\%} compared to IFS-HRES.
Specifically, \NAME achieves an average RMSE reduction of \textbf{5.8\%} at 24 hours and \textbf{9.7\%} at 72 hours relative to the IFS-HRES baseline. Furthermore, \NAME demonstrates a clear performance edge over 0.25° outputs upscaled via Swin2SR~\cite{conde2022swin2sr}, particularly in short-range windows. While auto-regressive systems often suffer from variance loss and converge toward the ensemble mean over time, post-processing techniques are inherently limited by the information bottleneck of coarse-grained inputs. In contrast, \textbf{\NAME's end-to-end high-resolution framework circumvents these limitations by capturing fine-grained dynamics natively}, yielding a more physically consistent and detailed atmospheric representation.

\subsubsection{Predictive Skill in Extreme Scenarios}

To evaluate \NAME's reliability in high-impact scenarios, we assess its performance across three categories of extremes: categorical detection of intense moisture events, predictive accuracy for both relative (percentile-based) and absolute (threshold-based) extremes.

\begin{figure}[htbp]
    \centering
    \includegraphics[width=0.9\textwidth]{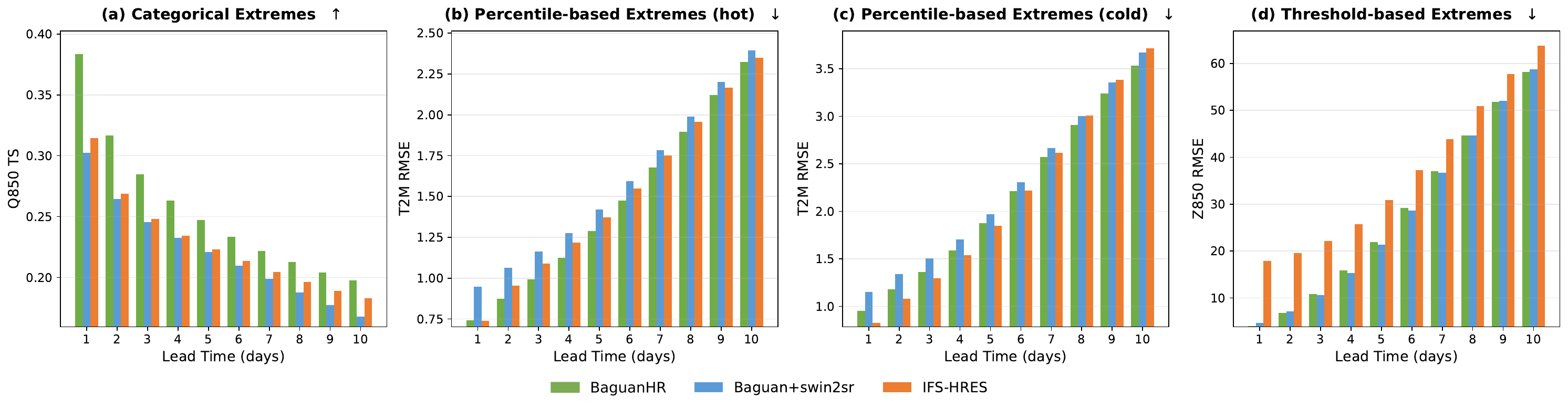}
    \caption{Extreme event verification. Higher TS represents better categorical skill; lower RMSE signifies enhanced accuracy for percentile and threshold-based extremes. 
    }
    \label{fig:extreme_analysis}
\end{figure}

\paragraph{Categorical Extreme Event Detection.}
Using q850 $\geq 14$ g/kg as a thermodynamic proxy for heavy precipitation~\cite{holton2013introduction}, we evaluate detection skill via the Threat Score (TS; Fig.~\ref{fig:extreme_analysis}(a)). \NAME consistently outperforms both Baguan+swin2sr and IFS-HRES, achieving gains exceeding 10\% in most cases. This confirms its superior ability to resolve the intense moisture convergence required for high-impact events.

\paragraph{Percentile-based Relative Extremes.}
To identify statistical anomalies across latitudes, we define extreme heat and cold using the local 90th and 10th percentiles of t2m~\cite{zhang2011indices} (Fig.~\ref{fig:extreme_analysis}(b--c)). \NAME outperforms competitors in more than 70\% of cases, surpassing Baguan+swin2sr by 15\% in the short term and maintaining a 3\% edge in long-term extremes. This highlights its precision in capturing the ``heavy-tail'' thermal distributions often smoothed by lower-resolution models.

\paragraph{Threshold-based Absolute Extremes.}
To evaluate dynamical intensity, we identify mid-latitude low-pressure systems using an absolute threshold (z850 $< 13{,}500$ m$^2$/s$^2$)~\cite{hoskins1985use} (Fig.~\ref{fig:extreme_analysis}(d)). \NAME delivers the best results, outperforming IFS-HRES by over 30\% and maintaining a significant lead over Baguan+swin2sr at both short (1–2 days) and long (9–10 days) lead times. These results demonstrate its capacity to reconstruct the sharp pressure gradients of severe storms.


\subsection{Controlled Comparisons}

We perform controlled comparisons to justify the core components of \NAME across four dimensions: variable modeling (independent vs. joint), physical consistency of SR fields, power spectral analysis, and optimization stability (replay buffer and adaptive weighting).

\textbf{Full-Variable vs. Per-Variable Super-Resolution}
  Fig.~\ref{fig:controlled_comparision} (a) shows the RMSE improvement of per-variable SR
  over full-variable SR across variables. It shows that per-variable SR significantly
  improves surface variables (e.g., 90.7\% for sp, >20\% for t2m and msl) by capturing
  local specificities. Upper-air fields achieve modest improvements (1.6\%--4.7\% for
  temperature and geopotential). This suggests that unlike forecasting, SR favors
  intra-variable textures over cross-variable coupling, making variable-independent
  modeling more effective.
\label{results:per_vs_full}
\begin{figure}[htbp]
    \centering
    \includegraphics[width=0.85\textwidth]{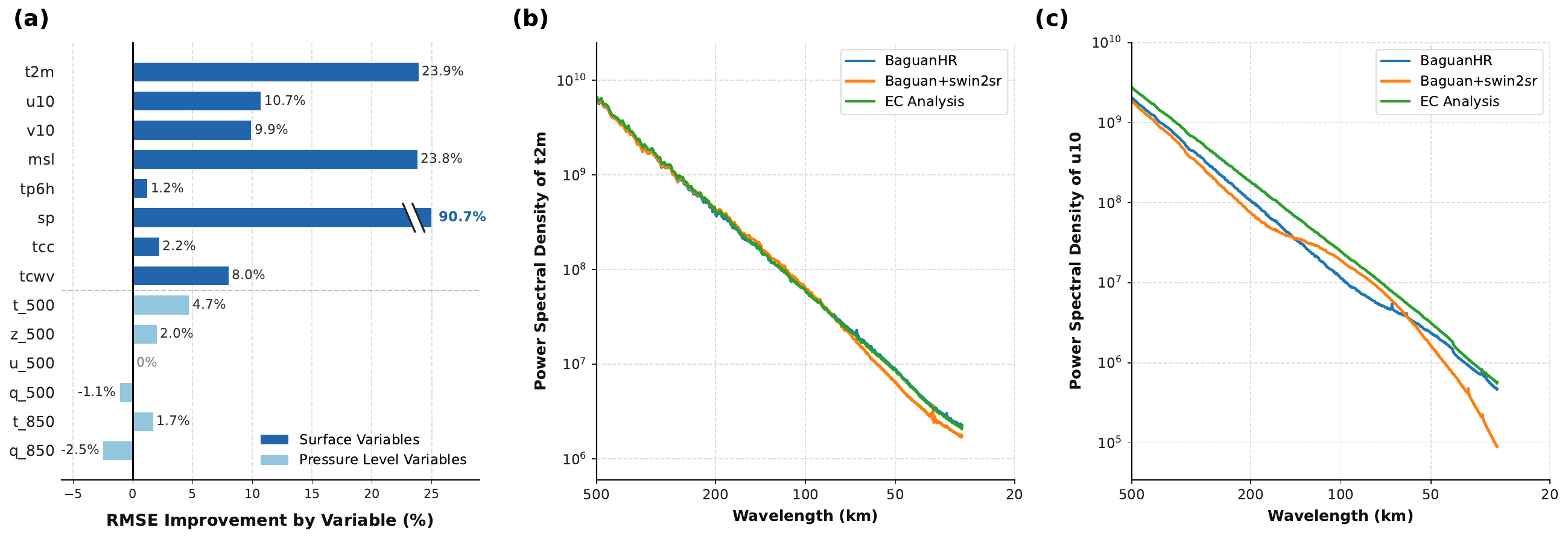}
    \caption{(a) Relative RMSE improvement of per-variable versus full-variable super-resolution. (b-c) Power spectral analysis of t2m and u10.}
    \label{fig:controlled_comparision}
\end{figure}

\textbf{Physical Consistency of SR Fields}
  We diagnose the physical consistency of SR-generated pseudo-labels using three
  diagnostics based on relative humidity (RH), geostrophic wind speed $|\mathbf{V}_g|$, and
  potential temperature differences $\Delta\theta$. Their latitude-weighted correlations
  against the 0.1$^\circ$ analysis remain high across pressure levels (\textbf{RH $\geq
  0.94$, $|\mathbf{V}_g| \geq 0.82$, $\Delta\theta \geq 0.93$}), indicating that
  per-variable SR preserves key moisture, dynamical, and thermodynamic structures,
  alleviating concerns about physically inconsistent high-frequency artifacts. Per-level
  numbers and the power-spectrum analysis are in Supplementary Sec.~E.3.

\textbf{Power Spectral Analysis}
We compute power spectral density (PSD; variance/energy as a function of spatial wavelength) for 6--72h forecasts over the full 2025 validation dataset; EC analysis denotes the corresponding year-mean PSD of analysis data. The results in 
Fig.~\ref{fig:controlled_comparision} (b-c) demonstrate that \NAME maintains higher short-wavelength spectral power, effectively preserving mesoscale structures. Conversely, the Baguan+swin2sr pipeline exhibits systematic energy attenuation, suggesting that post-hoc SR tends to smooth fields rather than recovering fine-scale variability from coarse inputs.


\textbf{Self-Adaptive Rollout Strategy}
  Tab.~\ref{tab:abalation_on_rb} validates our enhanced replay buffer strategy. Loss
  weighting suppresses auto-regressive error propagation, as its removal leads to degraded
  performance on long-range forecasts. Stochastic replacement further boosts performance by
  ensuring lead-time diversity. Together, \NAME harmonizes divergent loss magnitudes
  across horizons, ensuring stable optimization and superior metrics across all variables.
\input{tables/1_replay_buffer}
\subsection{Case Studies}

The aforementioned capabilities of \NAME are further demonstrated through the case studies of tropical cyclone and cold-air outbreak prediction. 

\begin{figure}[htbp]
    \centering
    \includegraphics[width=1\textwidth]{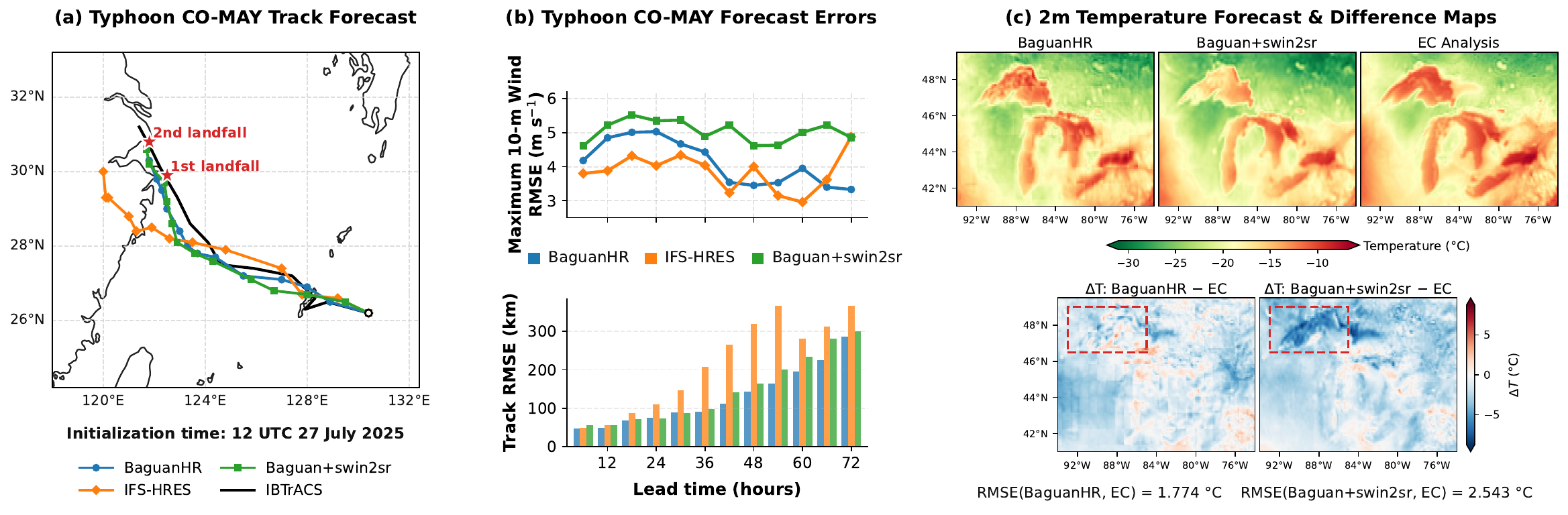}
    \caption{(a) Typhoon CO-MAY tracking forecast, which is initialized at 12 UTC, 27 July 2025. (b) Forecast errors of track and intensity. (c) Case study of 2m temperature forecast over the U.S. Great Lakes region, initialized at 00 UTC, 20 January 2025.} 
    \label{fig:case_study}
\end{figure}

\textbf{Tropical Cyclone Prediction} Tropical cyclones are among Earth's most hazardous weather events. While prior ML methods
  typically lag behind the IFS-HRES in intensity \cite{pangu_nature, Bodnar2024AuroraAF},
  \NAME achieves comparable intensity while maintaining superior track accuracy. For the
  double-landfall cyclone CO-MAY, \NAME precisely identifies two landfall sites,
  outperforming the IFS-HRES track (Fig.~\ref{fig:case_study}(a)). Furthermore, as shown in
  Fig.~\ref{fig:case_study}(b), \NAME substantially outperforms Baguan+swin2sr in both
  intensity and track errors.

\textbf{Cold-Air Outbreak} Fig.~\ref{fig:case_study}(c) presents a 2m temperature (t2m) forecast case study in a cold-air outbreak over the U.S. Great Lakes region. \NAME successfully captures the spatial structure characteristics, e.g., the distinct temperature characteristics over lake surfaces compared to surrounding land areas, a level of physical detail that the post-hoc Baguan+swin2sr baseline fails to resolve, demonstrating its superior performance in extreme value prediction.


%% file: tables/1_replay_buffer.tex
\begin{table}[h]

    \caption{Performance comparison of different strategies on replay buffer.}
    \label{tab:abalation_on_rb}
    
    \centering
    \resizebox{0.9\textwidth}{!}{
        \begin{tabular}{c|cc|cc|cc|cc} 
            \hline 
            
            \multirow{2}{*}{Methods} & 
            \multicolumn{2}{c|}{T2M-72h} & 
            \multicolumn{2}{c|}{U10-72h} & 
            \multicolumn{2}{c|}{MSL-72h} &
            \multicolumn{2}{c}{TCC-72h}\\
            
            & RMSE$\downarrow$ & ACC$\uparrow$ 
            & RMSE$\downarrow$ & ACC$\uparrow$ 
            & RMSE$\downarrow$ & ACC$\uparrow$
            & RMSE$\downarrow$ & ACC$\uparrow$\\
            \hline
            
            w/o Random \& loss weight     & 1.12 & 0.90 & 1.64 & 0.91 & 147.80 & 0.980 & 0.289 & 0.61  \\
            w/o Random & 1.12 & 0.91 & 1.63 & 0.91 & 147.41 & 0.980 & 0.288 & 0.62 \\
            w/o loss weight & 1.11 & 0.91  & 1.63 & 0.92 & 145.32 & 0.981 & 0.289 & 0.62 \\
            \hline 
            
            \NAME   & \textbf{1.09} & \textbf{0.92} & \textbf{1.62} & \textbf{0.92} & \textbf{144.57} & \textbf{0.981} & \textbf{0.287} & \textbf{0.64} \\
            \hline 
            
        \end{tabular}
    }
\end{table}

%% file: sections/5_conclusion.tex
In this paper, we present \NAME, a $0.1^{\circ}$ global weather forecasting model
built on a data-transfer framework that leverages super-resolution-generated synthetic data. We demonstrate fundamental data scaling laws in high-resolution
forecasting, showing that model performance follows a power-law improvement
with training data volume. Using synthetic $0.1^{\circ}$ data, \NAME
achieves superior performance across nearly all variables, especially in the
first 72 hours, outperforming both IFS-HRES and SR-enhanced $0.25^\circ$
models. At 10--15 day lead times, \NAME brings limited gains as forecasts are
increasingly controlled by large-scale atmospheric patterns and fine-scale
details become weakly predictable. Building on Baguan, we are developing the NJU-Earth series as an evolving model family for global weather forecasting, with BaguanHR as its high-resolution component. This work will be extended toward ensemble forecasting and operational applications such as tropical cyclone intensity prediction.

%% file: sections/6_acknowledge.tex
This work was supported by the National Natural Science Foundation of China
(U2342218), the Fundamental and Interdisciplinary Disciplines Breakthrough
Plan of the Ministry of Education of China (JYB2025XDXM907), the DAMO Academy Research Intern Program, and the High Performance Computing Center, Nanjing University (NJU).

%% file: sections/appendix.tex
\section{Theoretical Justification of Synthetic Data Transfer}
\input{sections/app1_theory}

\section{Data and Evaluation Details}
\input{sections/app2_dataset}

\section{Model Architecture and Implementation Details}
\input{sections/app3_model_implementations}

\section{Large-Scale System Engineering and Computational Cost}
\label{app:io_and_cost}
\input{sections/app4_IO_strategy}

\section{Extended Discussions}
\input{sections/app5_extended_discussions}

\section{Limitations and Broader Impact}
\input{sections/app6_limitations}


%% file: sections/app1_theory.tex
\def \E {\mathrm{E}}
\def \x {\mathbf{x}}
\def \g {\mathbf{g}}
\def \L {\mathcal{L}}
\def \D {\mathcal{D}}
\def \z {\mathbf{z}}
\def \u {\mathbf{u}}
\def \H {\mathcal{H}}
\def \w {\mathbf{w}}
\def \R {\mathbb{R}}
\def \S {\mathcal{S}}
\def \regret {\mbox{regret}}
\def \Uh {\widehat{U}}
\def \Q {\mathcal{Q}}
\def \W {\mathcal{W}}
\def \N {\mathcal{N}}
\def \A {\mathcal{A}}
\def \q {\mathbf{q}}
\def \v {\mathbf{v}}
\def \M {\mathcal{M}}
\def \c {\mathbf{c}}
\def \ph {\widehat{p}}
\def \d {\mathbf{d}}
\def \p {\mathbf{p}}
\def \q {\mathbf{q}}
\def \db {\bar{\d}}
\def \dbb {\bar{d}}
\def \I {\mathcal{I}}
\def \f {\mathbf{f}}
\def \a {\mathbf{a}}
\def \b {\mathbf{b}}
\def \ft {\widetilde{\f}}
\def \bt {\widetilde{\b}}
\def \h {\mathbf{h}}
\def \B {\mathbf{B}}
\def \bts {\widetilde{b}}
\def \fts {\widetilde{f}}
\def \Gh {\widehat{G}}
\def \bh {\widehat{b}}
\def \fh {\widehat{f}}
\def \vb {\bar{v}}
\def \zt {\widetilde{\z}}
\def \zts {\widetilde{z}}
\def \s {\mathbf{s}}
\def \gh {\widehat{\g}}
\def \vh {\widehat{\v}}
\def \Sh {\widehat{S}}
\def \rhoh {\widehat{\rho}}
\def \hh {\widehat{\h}}
\def \C {\mathcal{C}}
\def \V {\mathcal{V}}
\def \t {\mathbf{t}}
\def \xh {\widehat{\x}}
\def \Ut {\widetilde{U}}
\def \wt {\widetilde{w}}
\def \Th {\widehat{T}}
\def \Ot {\tilde{\mathcal{O}}}
\def \X {\mathcal{X}}
\def \nb {\widehat{\nabla}}
\def \K {\mathcal{K}}
\def \P {\mathbb{P}}
\def \T {\mathcal{T}}
\def \F {\mathcal{F}}
\def \ft{\widetilde{f}}
\def \xt {\widetilde{x}}
\def \Rt {\mathcal{R}}
\def \V {\mathcal{V}}
\def \Rb {\bar{\Rt}}
\def \wb {\bar{\w}}
\def \fh {\widehat{f}}
\def \wh {\widehat{\w}}
\def \lh {\widehat{\lambda}}
\def \e {\mathbf{e}}
\def \B {\mathcal{B}}
\def \P {\mathcal{P}}
\def \vb {\bar{v}}
\def \ub {\bar{u}}
\def \Rt {\mathcal{R}}
\def \wh {\widehat{w}}
\def \Lh {\widehat{\L}}
\def \rh {\widehat{r}}
\def \G {\mathcal{G}}
\def \qh {\widehat{q}}
\def \Qh {\widehat{Q}}
\def \Gh {\widehat{G}}
\def \Ph {\widehat{P}}
\def \U {\mathcal{U}}
\def \diag {\text{diag}}
\def \Kh {\widehat{K}}
\def \ut {\tilde{u}}
\def \gb {\bar{\gamma}}
\def \gh {\widehat{g}}
\def \Vh {\widehat{V}}
\def \zh {\widehat{z}}
\def \zb {\bar{z}}
\def \vb {\bar{v}}
\def \pt {\widetilde{p}}
\def \Hh {\widehat{H}}
\def \sgn {\mbox{sgn}}
\def \eh {\hat{\epsilon}}
\def \Sh {\widehat{\Sigma}}
\def \St {\widetilde{\Sigma}}
\def \Vt {\widetilde{V}}
\def \tr {\mbox{tr}}
\def \yh {\widehat{y}}
\def \yt {\widetilde{y}}
\def \rh {\widehat{r}}
\def \Lt {\widetilde{\L}}
\def \Zt {\widetilde{Z}}
\def \muh {\widehat{\mu}}

\subsection{Introduction}
People have observed the phenomenon of model collapse when training a LM from iteratively generated data~\cite{shumailov-2024-model-collapse}. Later studies~\cite{dey-2024-universality,gerstgrasser-2024-avoid-model-collapse} showed that model collapse will not take place if we can ensure that the clean data is always included in every iteration of training. However, it is overall a consensus that learning from synthetic data will result in worse generalization error compared to learning from the clean data~\cite{amin-2026-limitation-of-ERM,gerstgrasser-2024-avoid-model-collapse,dey-2024-universality}. In this note, we will show that this is not always the case. We demonstrate that with appropriate procedure of generating synthetic data, we are able to significantly reduce the generalization error for linear regression. We note that our study is different from~\cite{gerstgrasser-2024-avoid-model-collapse} where the introduction of synthetic data is resulted in an increase in variance and therefore a worse regression error. In our study, we focused on the case of synthetic data where a simple model is learned to predict additional features that lead to a better prediction model. 
\subsection{Analysis}
We first describe the data generation process. To create an input pattern $x$, we first choose an index $k \in [m]$ from $1$ to $m$ uniformly at random, and sample $u \in \R^d$ the $k$th Gaussian distribution $\N(\mu_k, \sigma^2 I_d)$. Then, the final input pattern $x \in \R^{md}$ is computed as $x = u\otimes e_k$, where $e_k \in \{0, 1\}^m$ is a one-hot vector with its $k$th element being $1$. Finally, the corresponding output value $y$ is given as $y = \langle u, w_k \rangle + z = \langle w, x\rangle + z$, where $w = (w_1^{\top}, \ldots, w_m^{\top})^{\top} \in \R^{md}$ and $z \sim \N(0, \gamma^2)$. Let $\D_a = \left\{(x_i, y_i), i=1, \ldots, n\right\}$ be the set of training examples that are generated by the above procedure. For the convenience of our analysis, we assume that (a) $|\mu_k| = \sqrt{d}, \forall k \in [m]$, (b) $\frac{1}{m}\sum_{k=1}^m \mu_k\mu_k^{\top} = I$, (c) $\min\limits_{j \neq k} |\mu_k - \mu_j| = \Omega(\sqrt{d})$
and (d) $\sigma = o(1)$.

Using the standard regression model, we obtain an estimation of $w$, denoted as $\wh$, i.e
\begin{eqnarray}
\wh = \left(\frac{1}{n}\sum_{i=1}^n x_ix_i^{\top}\right)^{-1}\left(\frac{1}{n}\sum_{i=1}^n x_i y_i\right) \label{eqn:wh}
\end{eqnarray}
The following theorem shows the estimation error for $\wh$ in (\ref{eqn:wh}).
\begin{thm}
With a probability at least $1 - \delta$, we have
\[
\left|w - \wh\right| \leq O\left(\gamma\sqrt{\frac{dm}{n}\log\frac{1}{\delta}}\right)
\]
\end{thm}
We skip its proof as it is standard to obtain the above bound using the concentration inequality. 

{\justifying
Next, we assume that we have additional $N$ noisy training examples
$\mathcal{D}_{b}=\{(u_i,y_i)\}_{i=1}^N$, where $N \gg n$ and $u_i \in \mathbb{R}^d$.
Each training example $(u_i,y_i)$ in $\mathcal{D}_b$ is generated by the same
procedure as described above, but only $u_i$ (the $d$-dimensional vector sampled
from the selected Gaussian distribution) is saved, and the corresponding index
information—i.e., which Gaussian is used for sampling $u_i$—is missing.
Because of the lost information, we refer to examples in $\mathcal{D}_b$ as noisy
training examples. For convenience, we denote by $i_k$ the index of the Gaussian
distribution used to generate $u_i$. As indicated by the theorem below, a linear
regression model directly learned from $\mathcal{D}_b$ yields a significantly larger
regression error than $\widehat{w}$ in~(\ref{eqn:wh}).
}

\begin{thm}
Let $w_u \in \R^d$ be the linear regression model we learned from $\D_b$. Assume $N = +\infty$. We have
\[
\E_{k \in [m]}\E_{u \sim \N(\mu_k, \sigma^2 I_d)} \left[\left|\langle u, w_u\rangle - \langle u, w_k\rangle\right|^2\right] \geq \sigma^2\mbox{VAR}\left(w_1, \ldots, w_m\right)
\]
where
\[
\mbox{VAR}\left(w_1, \ldots, w_m\right) = \E_{k \in [m]}\left[\left|w_k - \E_j\left[w_j\right]\right|^2\right]
\]
\end{thm}
\begin{proof}
Given that $N$ is infinite large, we can replace the average in the expression of linear regression with expectation, leading to the following expression for $w_u$
\[
w_u = \left(\E[u_iu_i^{\top}]\right)^{-1}\E[u_iy_i] = \frac{1}{\left(1 + \sigma^2\right)m}\sum_{k=1}^m \mu_k\mu_k^{\top} w_k
\]
We then compute the expected regression error, i.e.
\begin{align*}
\mathbb{E}\!\left[\left(\langle u_i, w_u\rangle - \langle u_i, w_{i_k}\rangle \right)^2\right]
&= \mathbb{E}\!\left[\left(\langle u_i, w_u - w_{i_k}\rangle \right)^2\right] \\
&= \frac{1}{m}\sum_{k=1}^m \Big(\sigma^2\lVert w_u - w_k\rVert^2
      + \big|\langle \mu_k,\, w_u - w_k\rangle \big|^2\Big) \\
&\ge \frac{\sigma^2}{m}\sum_{k=1}^m \lVert w_u - w_k\rVert^2 \\
&\ge \frac{\sigma^2}{m}\sum_{k=1}^m \left\lVert w_k - \frac{1}{m}\sum_{j=1}^m w_j\right\rVert^2 \\
&= \sigma^2\,\mathrm{VAR}(w_1,\ldots,w_m).
\end{align*}
\end{proof}
Compared to $\wh$ in (\ref{eqn:wh}), whose regression error is given by
\[
\E\left[\left|\langle x_i, \wh - w\rangle\right|^2\right] \leq O\left(\frac{\gamma^2 dm}{n}\log\frac{1}{\delta}\right)
\]
Thus, when
\[
n \geq \Omega\left(\frac{\gamma^2 d m}{\sigma^2 \mbox{VAR}(w_1, \ldots, w_m)}\log\frac{1}{\delta}\right)
\]
we will have $\wh$ yields a smaller regression error than $w_u$, even when $N \rightarrow +\infty$.

Our second approach is to learn from $\D_a$ a model to predict the index of Gaussian distribution, and apply the learned prediction model to augment the input patterns with the predicted index. To this end, we will simply estimate the center of each Gaussian distribution by averaging the input patterns assigned with the same index, i.e.
\begin{eqnarray}
\muh_k = \frac{\sum_{i=1}^n I\left(\left|M_k x_i\right| > 0\right)M_k x_i}{\sum_{i=1}^n I\left(\left|M_kx_i\rangle\right| > 0\right)}, \forall k \in [m]\label{eqn:muh}
\end{eqnarray}
where 
\begin{eqnarray}
M_k = \left(\underbrace{\mathbf{0}_{d\times d}, \ldots, \mathbf{0}_{d\times d}}_{\mbox{$k-1$ zero matrices}}, I_d, \underbrace{\mathbf{0}_{d\times d}, \ldots, \mathbf{0}_{d\times d}}_{\mbox{$m-k$ zero matrices}}\right) \label{eqn:M}
\end{eqnarray}
Using the mean vectors $\muh_k, k \in [m]$ estimated from data $\D_a$, we will find the index of Gaussian distribution for each training examples $(u_i, y_i) \in \D_b, i=1, \ldots, N$ as follows
\begin{eqnarray}
\hat{k}_i = \mathop{\arg\min}_{j \in [m]} \left|u_i - \muh_j\right| \label{eqn:k-hat}
\end{eqnarray}
Finally, for each training example $(u_i, y_i) \in \D_b$, we construct a new input pattern $\xt_i = u_i \otimes e_{\hat{k}_i}$, and learn a linear regression model $\wt$ from the training examples $\D_c =\{(\xt_i, y_i), i=1, \ldots, N\}$. The following theorem bounds the recovery error for $\wt$.
\begin{thm}
Assume 
\[
n \geq md\log\frac{N}{\delta}
\]
Then, with a probability at least $1 - \delta$, we have
\[
\left|\wt - w\right| \leq O\left(\sqrt{\frac{md}{N}\log\frac{1}{\delta}}\right)
\]
\end{thm}
Comparing Theorem 3 to Theorem 1, we can see that the solution estimated the synthetical dataset $\D_c$ is significantly closer to the oracle $w$ than the one that is learned from the clean dataset $\D_a$. 

\subsection{Proof of Theorem 3}
\begin{proof}
First, since $x_i = u_i \otimes e_{k_i}$, we have $|M_kx_i| > 0$ if and only if $k = k_i$. Hence, the expression for $\muh_k$ in (\ref{eqn:muh}) computes the average of $u_i \in \R^d$ from $\D_a$ that are sampled from the $k$th Gaussian distribution $\N(\mu_k, \sigma^2 I_d)$. We denote by $m_k$ the number of training examples in $\D_a$ that are sampled from $\N(\mu_k, \sigma^2 I_d)$. Using the chernoff bound, it is easy to show that, with a probability at least $1 - \delta$, for any $k \in [m]$, we have
\[
\left|m_k - \frac{n}{m}\right| \leq O\left(\frac{n}{m}\log\frac{m}{\delta}\right)
\]
Thus, when $n \geq \Omega(m\log(m/\delta))$, we have
\[
m_k \geq \Omega\left(\frac{n}{m}\right)
\]
Then, using the concentration inequality for vectors, we have, with a probability $1 - \delta$,
\[
\left|\muh_k - \mu_k\right| \leq O\left(\frac{\sigma d}{m_k}\log\frac{m}{\delta} + \sigma\sqrt{\frac{d}{m_k}\log\frac{m}{\delta}}\right) = O\left(\sigma\sqrt{\frac{md}{n}\log\frac{m}{\delta}}\right)
\]

Since $\left|u_i - \mu_{k_i}\right|^2 \sim \sigma^2 \chi^2_d$, with a probability $1 - \delta$, 
\[
\left|u_i - \mu_{k_i}\right|^2  = \sigma^2 \left(d + O\left(\sqrt{d\log\frac{N}{\delta}}\right)\right), \forall i [N]
\]
we have
\begin{eqnarray*}
\max_{i \in [N]} \left|u_i - \muh_{k_i}\right| & \leq & \max_{i \in [N]} \left|\mu_i - \mu_{k_i}\right| + \max\limits_{j \in [m]}\left|\mu_j - \muh_j\right| \\
& \leq & \sigma\sqrt{d + O\left(\sqrt{d\log\frac{N}{\delta}}\right)} + O\left(\sigma\sqrt{\frac{md}{n}\log\frac{m}{\delta}}\right) \leq O\left(\sigma\sqrt{d}\right)
\end{eqnarray*}
and
\[
\min\limits_{i \in [N], k \neq k_i}\left|u_i - \muh_{k}\right| = \min\limits_{j \neq k} \left|\mu_j - \mu_k\right| - \max\limits_{i \in [N]}\left|u_i - \muh_{k_i}\right| = \Omega(\sqrt{d})
\]
As a result, when $\sigma = o(1)$, we have, with a probability $1 - \delta$, 
\[
\hat{k}_i = k_i, \; \forall i \in [N]
\]
indicating that we obtain the correct indices of Gaussian distributions for all the training examples in $\D_b$. Using the augmented training dataset $\D_c$, we learn a linear regression model $\wt \in \R^{md}$ as
\[
\wt = \left(\frac{1}{N}\sum_{i=1}^N \xt_i \xt_i^{\top}\right)^{-1}\left(\frac{1}{N}\sum_{i=1}^N \xt_i y_i\right)
\]
Since $\hat{k}_i = k_i$ for all $i \in [N]$, we have
\[
\wt = \left(\frac{1}{N}\sum_{i=1}^N x_i x_i^{\top}\right)^{-1}\left(\frac{1}{N}\sum_{i=1}^N x_i y_i\right)
\]
where $x_i = u_i \otimes e_{k_i}$. Following the same analysis, we have, with a probability at least $1 - \delta$, 
\[
\left|\wt - w\right| \leq O\left(\sqrt{\frac{md}{N}\log\frac{1}{\delta}}\right)
\]
\end{proof}

%% file: sections/app2_dataset.tex
\subsection{Dataset}

As illustrated in Table~\ref{tab:dataset}, the dataset comprises real-time analysis data and synthetic data generated through super-resolution applied to the ERA5 reanalysis \cite{hersbach2020era5}. 

We select 13 standard vertical pressure levels (50, 100, 150, 200, 250, 300, 400, 500, 600, 700, 850, 925, and 1000 hPa) and include key atmospheric variables at these levels: geopotential, temperature, u component of wind, v component of wind, and specific humidity. In addition to widely used surface level variables, such as 2-meter temperature (T2M), 2-meter dewpoint temperature (D2M), 10-meter wind components (U10, V10), and mean sea-level pressure (MSL), we also incorporate several specialized surface variables that enhance specific forecasting applications. These include surface pressure (SP), essential for tracking synoptic-scale weather systems and mass conservation; 100-meter wind components (U100, V100), which are valuable for wind power forecasting; low cloud cover (LCC), total cloud cover (TCC), surface solar radiation downwards (SSRD), and surface direct solar radiation (FDIR), which support solar power forecasting; and sea surface temperature (SST), which is critical for typhoon analysis and marine shipping applications. Furthermore, total column water (TCW), total column vertically-integrated water vapour (TCWV), and total precipitation (TP) are included to support precipitation forecasting.

\subsubsection{Real-time Analysis Data}
Real-time analysis data with high resolution of $0.1^{\circ}$, produced every 6 hours via 4D-Var data assimilation, serves as the initial condition for IFS-HRES forecasts. This analysis is generated by assimilating a wide range of observational sources, including satellite retrievals, radiosonde measurements, and surface station reports, thereby providing a dynamically consistent and physically coherent representation of the atmospheric state. Similarly, most AI-based models use upscaled $0.25^{\circ}$ data as input for real-time forecasting. The high-resolution analysis data has been provided by the European Centre for Medium-Range Weather Forecasts (ECMWF) since mid-2016. For the accumulated variables FDIR, SSRD, and TP, which are not available in the analysis data, we use IFS-HRES forecasts as substitutes. Specifically, FDIR and SSRD are taken as 1-hour accumulations over the past hour, while TP is provided in two variants: TP1H, representing 1-hour accumulations over the past hour, and TP6H, representing 6-hour accumulations over the past 6 hours. Note that the baseline operational IFS-HRES data utilized for evaluation corresponding to the year 2025 follows the active ECMWF operational cycles during that period (primarily CY49R1, transitioning to subsequent updates as implemented).

\subsubsection{Synthetic Data by Super-resolution}
The input for super-resolution is ERA5, a global atmospheric reanalysis dataset produced by the ECMWF. ERA5 provides detailed information on Earth’s climate and weather conditions from 1940 to the present. While it offers a spatial resolution of $0.25^{\circ}$, it significantly exceeds real-time analysis data in volume. Therefore, to leverage this rich resource, we feed ERA5 into the per-variable super-resolution model to generate synthetic high-resolution data at $0.1^{\circ}$, matching the resolution of the real-time analysis. The input variables are drawn from ERA5 and correspond to the same set of variables listed in Table~\ref{tab:dataset}, ensuring consistency between the synthetic and real training data.

\input{tables/0_dataset}

\subsubsection{Data Partitioning Strategy}
To ensure a robust evaluation and strictly prevent temporal data leakage, we partition our dataset chronologically across different stages:
\begin{itemize}
    \item \textbf{Super-Resolution Training:} The super-resolution model is trained using paired ERA5 and real-time $0.1^{\circ}$ analysis data from 2018 to 2024.
    \item \textbf{Synthetic Data Generation:} Using the trained super-resolution model, we synthesize 10 years of $0.1^{\circ}$ data from ERA5 covering 2007 to 2016. Due to storage limitations, only 10 years of data are generated, although using more data would likely yield greater benefits.
    \item \textbf{Forecasting Model Training:} The downstream \NAME forecasting model is trained on a combination of the 10-year synthetic dataset (2007--2016) and the real-time analysis data (2017--2024).
    \item \textbf{Evaluation:} Data from the entire year of 2025 is strictly reserved for model evaluation on unseen future states.
\end{itemize}

\subsection{Evaluation Metrics}

We adopt the standard weather evaluation methodology outlined in WeatherBench \cite{rasp2024weatherbench2benchmarkgeneration}, focusing on forecasts initialized at 00/12 UTC for the year 2025. There are two metrics: latitude-weighted Root Mean Square Error (RMSE) and latitude-weighted Anomaly Correlation Coefficient (ACC), which have also been widely used in ML-based models \cite{niu2025utilizing, graphcast, Fuxi_nature, pangu_nature}.

\subsubsection{Root mean square error (RMSE)}
RMSE is a commonly utilized statistical metric in geospatial analysis and climate science to evaluate the precision of a model's predictions or estimates over various latitudinal ranges:
\begin{equation}
    \label{rmse_metric}
    RMSE = \frac{1}{N} \Sigma_{t=1}^N \sqrt{\frac{1}{H \times W}\Sigma_{i=1}^H \Sigma_{j=1}^W L(i)(\tilde{X}_{t, i, j} - X_{t, i, j})^2}.
\end{equation}
The latitude weighting factor $L(i)$ is used to account for the differences in surface area represented by various latitudes on a spherical Earth. It is given by:
\begin{equation}
    \label{lat_weight}
    L(i) = \frac{\mathrm{cos}(\mathrm{lat}(i))}{\frac{1}{H} \Sigma_{k=1}^H\mathrm{cos}(\mathrm{lat}(k))}.
\end{equation}

\subsubsection{Anomaly correlation coefficient (ACC)}
ACC measures the spatial correlation between the prediction anomalies $\tilde{X}'$ and ground truth anomalies $X'$, both computed relative to the historical climatology. The ACC is evaluated spatially at each forecasting time step $t$ and then averaged over all $N$ valid times in the test set:
\begin{equation}
    \label{eq:acc}
    ACC = \frac{1}{N} \sum_{t=1}^{N} \frac{\sum_{i=1}^{H} \sum_{j=1}^{W} L(i) \tilde{X}'_{t,i,j} X'_{t,i,j}}{\sqrt{ \left[ \sum_{i=1}^{H} \sum_{j=1}^{W} L(i) (\tilde{X}'_{t,i,j})^2 \right] \left[ \sum_{i=1}^{H} \sum_{j=1}^{W} L(i) (X'_{t,i,j})^2 \right] }},
\end{equation}
where the anomalies are defined as:
\begin{equation}
    \label{eq:acc_anomalies}
    \tilde{X}'_{t,i,j} = \tilde{X}_{t,i,j} - C_{t,i,j}, \quad X'_{t,i,j} = X_{t,i,j} - C_{t,i,j}.
\end{equation}
Unlike taking a simple temporal mean over the evaluated test set, the climatology $C_{t,i,j}$ here denotes the historical day-of-year mean for the specific location $(i, j)$ corresponding to the validation time $t$, following the standard WeatherBench~2 guidelines \cite{rasp2024weatherbench2benchmarkgeneration}.

%% file: tables/0_dataset.tex
\begin{table*}[h]
\caption{Variables used in the training. The "Type" column categorizes variables as static (non-time-varying) properties, time-varying, single-level (surface) properties, or time-varying atmospheric properties. The "Variable Name" and "Abbrev." columns denote the labels assigned by ECMWF, whereas the "ECMWF ID" column specifies the numerical identifier allocated by the organization. The "Input" and "Output" columns indicate whether a variable functions as an input to or an output from BaguanHR.}
\label{tab:dataset}
\begin{center}

\resizebox{1\textwidth}{!}{
\begin{tabular}{llccccc}
\toprule
Type & Variable name & Abbrev. & Input & Output & ECMWF ID \\
\midrule
Atmospheric & Geopotential & Z & $\checkmark$ & $\checkmark$ & 129 \\
Atmospheric & U wind component & U & $\checkmark$ & $\checkmark$ & 131 \\
Atmospheric & V wind component & V & $\checkmark$ & $\checkmark$ & 132 \\
Atmospheric & Temperature & T & $\checkmark$ & $\checkmark$ & 130 \\
Atmospheric & Specific humidity & Q & $\checkmark$ & $\checkmark$ & 133 \\
\midrule
Single & 2-metre temperature & T2M & $\checkmark$ & $\checkmark$ & 167  \\
Single & 2-metre dewpoint temperature & D2M & $\checkmark$ & $\checkmark$ & 168 \\
Single & 10-metre U wind component & U10 & $\checkmark$ & $\checkmark$& 165 \\
Single & 10-metre V wind component & V10 & $\checkmark$ & $\checkmark$& 166 \\
Single & 100-metre U wind component & U100 & $\checkmark$ & $\checkmark$& 228246 \\
Single & 100-metre V wind component & V100 & $\checkmark$ & $\checkmark$& 228247 \\
Single & mean sea-level pressure & MSL & $\checkmark$ & $\checkmark$& 151 \\
Single & total cloud cover & TCC & $\checkmark$ & $\checkmark$& 164 \\
Single & low cloud cover & LCC & $\checkmark$ & $\checkmark$& 186 \\
Single & surface solar radiation downward & SSRD & $\checkmark$ & $\checkmark$& 169 \\
Single & total sky direct solar radiation at surface & FDIR & $\checkmark$ & $\checkmark$& 228021 \\
Single & sea surface temperature & SST & $\checkmark$ & $\checkmark$& 34 \\
Single & total column water & TCW & $\checkmark$ & $\checkmark$& 136 \\
Single & total column water vapor & TCWV & $\checkmark$ & $\checkmark$& 137 \\
Single & total precipitation & TP & $\times$ & $\checkmark$& 228 \\
\midrule
Static & Land-sea mask & LSM & $\checkmark$ & $\times$ & 172 \\
Static & geopotential at surface & - & $\checkmark$ & $\times$ & 162051 \\
Static & angle of sub gridscale orography & ANOR & $\checkmark$ & $\times$ & 162  \\
Static & anisotropy of sub gridscale orography & ISOR & $\checkmark$ & $\times$ & 161 \\
Static & lake cover & CL & $\checkmark$ & $\times$& 26 \\
Static & soil type & SOILTYP & $\checkmark$ & $\times$ & 500205  \\

\bottomrule
\end{tabular}
}
\end{center}
\end{table*}

%% file: sections/app3_model_implementations.tex
\subsection{Problem Statement}
\NAME is an advanced data-driven framework designed to generate accurate, high-resolution global medium-range weather forecasts. \NAME takes the atmospheric state at time $t$ as input, represented as $X^{t} \in \mathbb{R}^{V \times H \times W}$, where $V$ denotes the number of atmospheric and surface variables considered in the input, while $H$ and $W$ represent the height (latitude) and width (longitude) of the grid, respectively. In our high-resolution scenario with a spatial resolution of $0.1^\circ$, $H = 1801$ and $W = 3600$ for the Earth. \NAME aims to establish the mapping function $\tilde{X}^{t+\Delta t} = \text{\NAME}(X^{t})$, where $\Delta t$ denotes the time interval (6 hours). The model then produces 6-hourly forecasts into the future in an auto-regressive manner.

\noindent \textbf{System Overview.} To achieve this goal under the severe scarcity of high-resolution training data, our overall pipeline is decoupled into two critical components: (1) a \textbf{Super-Resolution Data Generator}, which first synthesizes extensive $0.1^{\circ}$ pseudo-labels from decades of coarse historical data; and (2) the \textbf{\NAME Forecasting Model}, which is subsequently trained from scratch on this augmented dataset to perform the core auto-regressive forecasting task. We detail both components below.

\subsection{Super-Resolution Data Generator (Swin2SR)}

Both data-driven models and numerical weather prediction systems fundamentally treat weather forecasting as an initial value problem governed by the Navier-Stokes equations and related atmospheric dynamics. In contrast, super-resolution (SR) is a spatial reconstruction task. It does not require full adherence to the temporal integration of physical equations but rather relies heavily on spatial priors and topographic information. As demonstrated by recent studies \cite{vandal2017deepsd, harris2022generative, addison2022machine}, training a separate SR model for individual variables effectively avoids cross-channel interference and yields significantly higher reconstruction quality. Thus, we employ a tailored Swin2SR \cite{conde2022swin2sr} architecture to perform variable-specific downscaling of the $0.25^\circ$ ERA5 data to $0.1^\circ$.

The architecture is specifically designed to leverage high-resolution static topographies. The forward process features two key components:
\begin{itemize}
    \item \textbf{Topographic-Conditioned Embedding:} The generator receives both the $0.25^\circ$ coarse-grained atmospheric variables and the $0.1^\circ$ high-resolution static topographical constants (e.g., land-sea mask, orography). Through a dual-branch convolutional stem with different strides, these two inputs are spatially aligned and fused into a unified latent space. This ensures that the generated atmospheric features are strictly conditioned on the underlying physical geography.
    \item \textbf{High-Frequency Residual Learning:} Following deep feature extraction by the Swin2SR backbone and PixelShuffle upsampling, the model avoids generating the absolute target values directly. Instead, it outputs a residual map $\Delta X$. The final $0.1^\circ$ pseudo-label is obtained by adding this learned high-frequency residual to the bicubically interpolated low-resolution input:
    \begin{equation}
        \hat{X}_{hr} = \mathrm{Interpolate}(X_{lr}, \text{mode='bicubic'}) + \Delta X
    \end{equation}
    This explicit residual connection forces the network to focus exclusively on reconstructing missing fine-scale textures, yielding high physical fidelity.
\end{itemize}

\subsection{The \NAME Forecasting Model}
The architecture of the \NAME forecasting model follows an encoder-backbone-decoder paradigm with a global residual connection. It effectively transforms physical states into a latent space, simulates atmospheric dynamics over time $\Delta t$, and projects the predicted increments back to the physical grid. Prior to processing, all input variables $X^t$ are standardized using \textbf{Z-score normalization} based on the historical climatological mean and standard deviation, ensuring stable gradient propagation.

\paragraph{Encoder: Variable-Specific Tokenization and Hierarchical Embedding.}
The encoder is responsible for mapping the high-dimensional, multi-variable physical inputs into a unified latent space. Given the input atmospheric state $X^t \in \mathbb{R}^{V_{in} \times H \times W}$ (where $V_{in}=86$), each variable is first independently tokenized via a patchify mechanism with a patch size of $p=9$. Subsequently, a learnable variable-specific embedding is added to each variable's token sequence to preserve its physical identity. Unlike natural images, atmospheric states involve tens of distinct physical variables. A naive cross-attention aggregation across all $V_{in}$ variables would result in an excessive computational and memory bottleneck. To address this, we propose a \textbf{Hierarchical Weather Embedding}, which utilizes a dual-stage cross-attention mechanism:
\begin{enumerate}
    \item \textbf{Group-Wise Aggregation:} The 86 variables are partitioned into 12 physical groups. For the five 3D upper-air variables ($z, u, v, t, q$), we create an \textit{overlapping split} based on pressure levels to maintain vertical continuity: levels between $50 \sim 600$ hPa form the ``upper-level'' groups, while levels between $400 \sim 1000$ hPa form the ``lower-level'' groups. This yields 10 atmospheric groups. The time-varying surface variables (e.g., t2m, u10, sp) and the static geographic variables (e.g., land-sea mask, orography) form the 11th and 12th groups. Within each group, a dedicated learnable channel query aggregates information via cross-attention, producing a group-specific summary vector.
    \item \textbf{Global Ensemble Aggregation:} An ensemble query then performs a secondary cross-attention over the 12 group-specific summaries. The resulting vectors are concatenated and projected to form the final multi-modal latent token of dimension $C = 1536$.
\end{enumerate}
Furthermore, to provide the model with essential seasonal and diurnal context, we inject a \textbf{Time Embedding} into the latent representation. The temporal features—the day of the year ($1$--$365$) and the hour of the day ($0$--$23$)—are passed through distinct learnable embedding dictionaries and added to the spatial tokens. An explicit full positional embedding is also applied to preserve global geographic coordinates.

\paragraph{Swin Transformer Backbone.}
If the embedded tensor that the backbone receives as input can be viewed as a latent grid on which the simulation is performed, then the backbone acts as a neural physics simulator. We opt for a Swin Transformer V2 \cite{liu2022swin} backbone with cross-shaped windows. The embeddings are passed through a deep stack of 24 Swin Transformer blocks to facilitate interactions among different patches. In contrast to the standard Vision Transformer (ViT), our choice of Swin Transformer is critical for $0.1^\circ$ resolution. Each layer performs local self-attention operations between tokens within $20 \times 40$ local windows. In alternating layers, the windows are shifted, inherently enforcing locality and emulating the message-passing computations performed by traditional numerical PDE solvers. This mitigates the issue of attention maps spreading uniformly without local structure, while strictly avoiding the quadratic complexity of vanilla Transformers.

\paragraph{Decoder and Global Residual Connection.}
After latent-space evolution by the backbone, the decoder projects the latent representation back to the original spatial dimensions. Rather than directly predicting the absolute atmospheric state at the next time step, the model predicts the dynamical increment over $\Delta t$, which is added to the current input fields through a global skip connection. Accumulated variables such as total precipitation are generated directly. This residual design constrains the prediction trajectory and improves the stability of auto-regressive forecasting. Finally, inverse Z-score normalization is applied to recover the physical atmospheric variables.

\subsection{Experimental Setups}

The network architecture hyperparameters and computational complexity of both the Super-Resolution model and the forecasting model are detailed in Table~\ref{tab:arch_hyperparams}. 

\paragraph{Objective Loss Function.}
The super-resolution model and \NAME forecasting model are both optimized using the \textbf{Latitude-weighted Mean Absolute Error (MAE)} as the objective loss function. This ensures that the optimization process accounts for the varying grid cell areas on the spherical Earth topology. The loss $\mathcal{L}$ is defined as:
\begin{equation}
    \mathcal{L}_{MAE} = \frac{1}{V \times H \times W} \sum_{v=1}^{V} \sum_{i=1}^{H} \sum_{j=1}^{W} L(i) \left| \tilde{X}^{t+\Delta t}_{v, i,j} - X^{t+\Delta t}_{v, i,j} \right|
\end{equation}
where $L(i) = \frac{\cos(\text{lat}(i))}{\frac{1}{H}\sum_{k=1}^H \cos(\text{lat}(k))}$ is the latitude weighting factor.

\paragraph{Training Stages and Configurations.}
The training pipeline encompasses the upstream data generation and the downstream forecasting model optimization. All models are trained on a cluster of 32 NVIDIA A800 GPUs. To ensure stable convergence and prevent error accumulation in the forecasting phase, the entire process is conducted in four progressive stages:

\begin{itemize}
    \item \textbf{Stage 0: Super-Resolution Data Synthesis.} Thanks to the lightweight architectural design of the Swin2SR model and the localized nature of the spatial reconstruction task, the per-variable SR models are highly efficient to optimize. Training the SR model for a single variable takes only $\sim24$ hours on 8 GPUs. Subsequently, synthesizing the entire 10-year (2007--2016) $0.1^\circ$ global dataset requires approximately 2 days of inference time across the cluster.
    
    \item \textbf{Stage 1: Pretraining on Analysis Data.} The \NAME forecasting model is pretrained from scratch for a single-step forecast ($\Delta t = 6$h) using the EC analysis data (2017--2024) exclusively. This stage spans approximately 14 days (250,000 optimization steps). The learning rate is warmed up linearly to $5 \times 10^{-4}$ and then decays to $1 \times 10^{-8}$ following a cosine annealing schedule.
    
    \item \textbf{Stage 2: Synthetic Fine-Tuning.} We extend the single-step training by incorporating the additional 10 years of synthetic $0.1^\circ$ data generated in Stage 0. This phase lasts approximately 4 days (70,000 steps) with a linearly warmed-up learning rate peaking at $2 \times 10^{-5}$.
    
    \item \textbf{Stage 3: Rollout Training.} Finally, to mitigate error accumulation over long horizons, we perform 2 days of auto-regressive multi-step fine-tuning using both the synthetic and real datasets. We employ a self-adaptive replay buffer and store up to 40 historical states per GPU worker, with the buffer size chosen according to strict CPU memory constraints. During this stage, the learning rate is fixed at $2 \times 10^{-6}$.
\end{itemize}

\begin{table}[htbp]
\centering
\caption{Architectural hyperparameters of both the Super-Resolution Data Generator (Swin2SR) and the \NAME forecasting model. Note that the original latitude dimension of 1801 is cropped to 1800 to ensure exact divisibility by the patch size (9×9).}
\label{tab:arch_hyperparams}
\resizebox{0.95\textwidth}{!}{
\begin{tabular}{lcc}
\toprule
\textbf{Hyperparameter}                 & \textbf{Super-Resolution (Swin2SR)} & \textbf{\NAME Forecasting Model} \\
\midrule
Input Resolution ($H \times W$)         & $720 \times 1440$ ($0.25^\circ$) & $1800 \times 3600$ ($0.1^\circ$) \\
Output Resolution ($H \times W$)        & $1800 \times 3600$ ($0.1^\circ$) & $1800 \times 3600$ ($0.1^\circ$) \\
Input Channels ($V_{in}$)               & 82 (Atmospheric) + 6 (Constants) & 80 (Atmospheric/Surface) + 6 (Constants) \\
Output Channels ($V_{out}$)             & 1 (Variable-specific) & 82 (Including Precipitation) \\
Patch Size ($p \times p$)               & $1 \times 1$ & $9 \times 9$ \\
Latent Embedding Dimension ($C$)        & 384 & 1536 \\
Swin Transformer Blocks                 & [6, 6, 6, 6] \textit{(4 stages)} & 24 \textit{(Single deep stage)} \\
Window Size                             & $12 \times 12$ & $20 \times 40$ \\
Attention Heads                         & [12, 12, 12, 12] & 16 \\
\midrule
\textbf{Total Parameters}               & 500.74 M & 821.65 M  \\
\bottomrule
\end{tabular}
}
\end{table}

\begin{table}[htbp]
\centering
\caption{Optimization configurations across the Super-Resolution (SR) generation stage and the progressive training stages of the \NAME forecasting model.}
\label{tab:training_setups}
\resizebox{\textwidth}{!}{
\begin{tabular}{lcccc}
\toprule
\textbf{Configuration} & \textbf{SR Model (Data Gen.)} & \textbf{\NAME Stage 1} & \textbf{\NAME Stage 2} & \textbf{\NAME Stage 3} \\
\midrule
Task Description & Spatial downscaling & Single-step forecast & Single-step forecast & Auto-regressive \\
Batch Size (Per GPU) & 4 & 1 & 1 & 1\\
Optimizer & AdamW & AdamW & AdamW & AdamW \\
$\beta_1, \beta_2$ & $0.9, 0.99$ & $0.9, 0.95$ & $0.9, 0.99$ & $0.9, 0.99$ \\
Weight Decay & $1 \times 10^{-5}$ & $0.05$ & $1 \times 10^{-5}$ & $1 \times 10^{-5}$ \\
Peak Learning Rate & $2 \times 10^{-4}$ & $5 \times 10^{-4}$ & $2 \times 10^{-5}$ & $2 \times 10^{-6}$ (Fixed) \\
LR Schedule & Cosine Annealing & Linear + Cosine & Linear + Cosine & Constant \\
Optimization Steps & 50,000 & 250,000 & 70,000 & 5,000 \\
\bottomrule
\end{tabular}}
\end{table}

\paragraph{Data Mixing and Sampling Strategy.}
During the Synthetic Fine-Tuning (Stage 2) and Rollout Training (Stage 3), the model learns from a combined dataset comprising the 2017-2024 real EC analysis data and 10 years of synthetic SR data. Prior to training, this mixed dataset is uniformly partitioned and sharded across the local NVMe SSDs of the compute nodes. During optimization, each GPU worker independently samples batches strictly from its local SSD shard. Given the comparable volume of the two data sources, this decentralized sampling strategy naturally yields a heterogeneous mini-batch composition where the ratio of real to synthetic samples is continuously maintained at approximately \textbf{1:1}. A detailed exposition of this I/O-optimized data loading framework is provided in Section \ref{app:io_and_cost}.

\subsection{Self-Adaptive Rollout and Replay Buffer Strategy}

In auto-regressive weather forecasting (Stage 3 rollout training), multi-step fine-tuning is necessary to reduce error accumulation at long lead times. While mainstream approaches often rely on simple unrolled backpropagation, recent models such as Fengwu \cite{han2024fengwu} and Aurora \cite{Bodnar2024AuroraAF} have introduced replay-buffer-based designs to improve long-range training stability. However, uniformly optimizing buffered long-horizon states can bias training toward long-range error correction and degrade short-range predictive fidelity.

To address this trade-off under limited GPU and CPU memory, we propose a \textbf{Self-Adaptive Rollout Strategy} with two components:

\paragraph{Lead-Time-Aware Loss Weighting.}
Errors typically grow with lead time in an unrolled trajectory, so treating all steps equally can cause long-range losses to dominate optimization. We therefore apply a decaying weight based on the auto-regressive step index. For the $k$-th prediction step ($k=1,2,\dots,K$), the loss is weighted by $1/k$, yielding
\begin{equation}
    \mathcal{L}_{\text{rollout}} = \sum_{k=1}^{K} \frac{1}{k}\,\mathcal{L}_{\text{MAE}}(\tilde{X}^{t_0 + k\Delta t}, X^{t_0 + k\Delta t}).
\end{equation}
This harmonic decay emphasizes short-range accuracy while still providing supervision for longer-horizon prediction.

\paragraph{Stochastic Buffer Replacement Schedule.}
Replay buffers in ML-based weather models are often implemented as First-In-First-Out (FIFO) queues. For a high-resolution model like \NAME, however, the large activation footprint limits the local batch size to very small values (often $N=1$ per GPU). Under distributed training, a FIFO buffer can therefore make the global batch overly concentrated on a narrow range of lead times, reducing trajectory diversity and destabilizing optimization. To alleviate this issue, we replace FIFO with a \textbf{Stochastic Buffer Update Strategy}. When a trajectory in the 40-sample CPU buffer reaches its maximum rollout horizon and is removed, we randomly insert initialization states from different historical timestamps. As a result, each optimization step draws a more heterogeneous mixture of atmospheric states and lead-time stages across the 32 GPU workers. This reduces the tendency of the global batch to be dominated by a single forecasting horizon and helps improve optimization stability during distributed training.

%% file: sections/app4_IO_strategy.tex
The transition from $0.25^\circ$ to $0.1^\circ$ global resolution poses unprecedented system-level challenges. The intrinsic precision requirements of meteorological data necessitate 32-bit storage, as 16-bit quantization introduces non-negligible numerical degradation across numerous thermodynamic variables. This high-precision representation of a dense $1801 \times 3600$ grid results in \textbf{exceptionally large individual data files—typically 4 to 5 GB each}. 

Consequently, training high-resolution weather forecasting models faces a severe I/O bottleneck, a challenge rarely encountered in the standard training pipelines of large language models or computer vision systems. The end-to-end training throughput is strictly bottlenecked by the storage-to-CPU bandwidth rather than the GPU's theoretical floating-point operations per second (FLOPs). To decouple GPU computation from disk latency and maximize hardware utilization, we employ several complementary I/O strategies.

\subsection{Hybrid I/O and Caching Strategies}

Through the following three customized data-loading strategies, our pipeline effectively masks the I/O latency and achieves \textbf{a training speedup of over $4\times$} compared to naive network storage access.

\paragraph{Data Partitioning by Machine.} 
Our storage infrastructure includes a large-scale Network Attached Storage (NAS) system alongside approximately 12 TB of high-speed NVMe SSD storage per 8-GPU machine. To leverage the significantly faster read/write performance of local SSDs compared to the NAS, we partition the entire training dataset into $n$ shards—where $n=4$ equals the number of physical machines—and distribute one unique shard to each machine. During distributed training, each machine loads only the data shard corresponding to its assigned node rank, drastically minimizing cross-node network I/O overhead and maximizing localized throughput.

\paragraph{Hybrid Storage Strategy Based on Time-ID Partitioning.} 
Despite the data partitioning, the aggregate 12 TB of SSD storage per machine is still insufficient to house the full historical dataset. To address this, we implement a time-based hybrid storage strategy: data initialized at UTC 00:00 and 12:00 (which constitute the primary cycles for operational weather forecasting) is permanently cached on the local SSDs for high-speed access. Meanwhile, data from the intermediate UTC 06:00 and 18:00 cycles resides on the NAS and is fetched dynamically. This specific allocation strategy guarantees a high cache hit rate on the SSDs while remaining within hardware limits.

\paragraph{Efficient Replay of Cached Data.} 
Our distributed training pipeline first transfers data from storage (SSD/NAS) to CPU memory before moving it to the GPU memory. Because the end-to-end throughput is bottlenecked by the sluggish I/O bandwidth from storage to CPU, we mitigate this latency through an in-memory replay mechanism. We aggressively reuse data batches currently resident in the CPU memory queue to feed the GPUs, maintaining active compute cycles whenever asynchronous I/O threads are still fetching new data from the storage layer.

\subsection{System-Level Profiling}

To accurately profile the isolated impact of our I/O optimization strategies, it is essential to decouple the storage-to-CPU bottleneck from inter-node network communication overhead (e.g., NCCL gradient synchronization across machines). Therefore, we benchmarked the training throughput on a single compute node (equipped with 8 NVIDIA A800 GPUs and local NVMe SSDs) representing one functional shard of our full 32-GPU distributed pipeline.

As detailed in Table~\ref{tab:io_performance}, the \NAME\ architecture exhibits exceptionally high computational complexity, reaching approximately 364.7 TFLOPs per optimization step per GPU. However, a baseline data-loading strategy that reads training samples directly from the network-attached storage (NAS) leads to severe GPU under-utilization. The average I/O wait time reaches 28.7 seconds per step, reducing the GPU active ratio to only 54.0\%. In contrast, our hybrid I/O and SSD partitioning strategy effectively mitigates this system bottleneck, reducing the average I/O wait time to just 0.69 seconds per step and increasing the GPU active ratio to 92.0\%. These results indicate that, for $0.1^\circ$ training, the dominant limitation is not raw computation itself, but the efficiency of the storage and data delivery pipeline.

\begin{table}[htbp]
\centering
\caption{System-level profiling of computational complexity and I/O efficiency under different data-loading strategies. To isolate the storage bottleneck from cross-node communication overhead, metrics were collected over 200 optimization steps on a single operational 8 $\times$ A800 GPU node.}
\label{tab:io_performance}
\resizebox{\textwidth}{!}{%
\begin{tabular}{l|c|cc|cc}
\toprule
\multirow{2}{*}{\textbf{Model Pipeline}} 
& \multirow{2}{*}{\shortstack{\textbf{Theoretical Compute}\\\textbf{(TFLOPs / Step / GPU)}}}
& \multicolumn{2}{c|}{\textbf{Baseline: Direct NAS Read}}
& \multicolumn{2}{c}{\textbf{Ours: Hybrid I/O \& Partitioning}} \\
\cmidrule(lr){3-4} \cmidrule(lr){5-6}
& & \textbf{GPU Active Ratio} & \textbf{I/O Wait / Step}
& \textbf{GPU Active Ratio} & \textbf{I/O Wait / Step} \\
\midrule
\NAME\ ($0.1^\circ$) & 364.7 & 54.0\% & 28,714 ms & \textbf{92.0\%} & \textbf{692 ms} \\
\bottomrule
\end{tabular}%
}
\end{table}

%% file: sections/app5_extended_discussions.tex
\subsection{Experimental Protocol for Conditional Entropy and Task Robustness}
\label{app:entropy_protocol}

In the main text (Section 3.1), we posit that ``Super-Resolution (SR) has provably lower conditional entropy than multi-step forecasting (FC)'', making it a significantly more robust vehicle for data transfer than model fine-tuning. Due to page constraints, the specific experimental procedures used to validate this core insight were omitted. Here, we detail the lightweight inference protocols designed to measure task difficulty and input perturbation sensitivity.

\subsubsection{Experiment I: Task Difficulty Comparison}
The primary objective of this experiment is to empirically demonstrate that SR is an intrinsically simpler task than temporal forecasting by directly comparing their reconstruction errors. 
\begin{itemize}
    \item \textbf{SR Task Protocol:} The model takes the low-resolution state $X_t^{0.25^\circ}$ as input and outputs the high-resolution reconstruction $\hat{X}_t^{0.1^\circ}$. The ground truth is the exact same-time high-resolution analysis $X_t^{0.1^\circ}$.
    \item \textbf{Forecasting Task Protocol:} The model takes the high-resolution state $X_t^{0.1^\circ}$ as input and predicts the future state $\hat{X}_{t+6h}^{0.1^\circ}$. The ground truth is the actual future analysis $X_{t+6h}^{0.1^\circ}$.
\end{itemize}
As reported in the main text, the SR task yields significantly lower RMSE (e.g., 13.8 vs. 23.6 for z500). From an information-theoretic perspective, the lower error indicates that the conditional entropy of same-time spatial reconstruction $H(X_t^{0.1^\circ} | X_t^{0.25^\circ})$ is substantially lower than that of cross-time prediction $H(X_{t+6h}^{0.1^\circ} | X_t^{0.1^\circ})$.

\subsubsection{Experiment II: Input Perturbation Sensitivity}
The second experiment aims to prove that forecasting, governed by chaotic atmospheric dynamics, amplifies input errors exponentially, whereas the SR mapping remains highly stable. 

We apply isotropic Gaussian noise $\varepsilon \sim \mathcal{N}(0, \sigma^2)$ at varying intensities $\sigma \in \{0.01, 0.05, 0.10, 0.20, 0.25\}$ to the input fields of both models. The sensitivity is quantified by the \textbf{Amplification Factor}, defined as the ratio of the output error change to the magnitude of the injected noise.
For the SR model:
\begin{equation}
    Y_{\text{clean}}^{SR} = \Phi_{SR}(X_t), \quad Y_{\text{noisy}}^{SR} = \Phi_{SR}(X_t + \varepsilon)
\end{equation}
\begin{equation}
    \text{Amplification}_{SR} = \frac{||Y_{\text{noisy}}^{SR} - Y_{\text{clean}}^{SR}||}{||\varepsilon||}
\end{equation}
For the Forecasting (FC) model:
\begin{equation}
    Y_{\text{clean}}^{FC} = \Phi_{FC}(X_t), \quad Y_{\text{noisy}}^{FC} = \Phi_{FC}(X_t + \varepsilon)
\end{equation}
\begin{equation}
    \text{Amplification}_{FC} = \frac{||Y_{\text{noisy}}^{FC} - Y_{\text{clean}}^{FC}||}{||\varepsilon||}
\end{equation}

As illustrated in Figure 1 of the main text, the $\text{Amplification}_{SR}$ curve remains relatively flat and strictly below $1.0$ (e.g., $0.6 \sim 0.8\times$), indicating that input information is transferred to the output with minimal and reversible loss. In sharp contrast, the $\text{Amplification}_{FC}$ curve exhibits an upward trend ($> 1.0\times$), demonstrating that small input perturbations are irreversibly magnified due to the chaotic nature of temporal evolution. This justifies our paradigm shift: generating pseudo-labels via SR is fundamentally more robust than attempting to extract fine-scale structures from a pre-trained $0.25^\circ$ forecasting model.

\subsection{Mathematical Formulation of the Data Scaling Law}

In main text Section 5.2 (Fig. 4), we presented the empirical data scaling laws, demonstrating that model performance improves as the volume of high-resolution training data increases, before eventually plateauing. To rigorously formalize this observation, we fitted the empirical validation RMSE to a generalized power-law function with an irreducible error term:
\begin{equation}
    E(D) = \alpha D^{-\beta} + E_{\infty}
\end{equation}
where $E(D)$ is the forecast error given training data volume $D$ (in years), $\alpha$ is a scaling constant, $\beta > 0$ is the characteristic scaling decay exponent, and $E_{\infty}$ represents the theoretical saturation bound (irreducible error). The precise fitted parameters are detailed in Table~\ref{tab:scaling_law_params}.

The emergence of the saturation threshold $E_{\infty}$ (evident as the curves flatten around the 18-year mark in Fig. 4) suggests that while synthesizing additional decades of data via SR improves the representation of rare meso-scale anomalies, the deterministic predictive skill is ultimately bounded by the inherent chaotic limits of atmospheric predictability and the parameter capacity of the current architecture. 

Crucially, it is important to note that the models evaluated in this scaling analysis were trained exclusively for the \textbf{6-hour single-step prediction task}, without undergoing the subsequent auto-regressive rollout fine-tuning. Consequently, the 6-hour forecast metrics (where $R^2 > 0.95$ across all variables) provide the purest, most direct reflection of the underlying data scaling mechanism, isolated from the complex error accumulation inherent in long-range unrolling. Furthermore, as revealed in Table~\ref{tab:scaling_law_params}, the scaling exponent $\beta$ exhibits a systematic decay as the forecast lead time increases (e.g., for T2M, $\beta$ drops from 2.207 at 6h to 0.792 at 120h). This compelling quantitative finding indicates that while short-term localized dynamics are highly data-hungry and benefit immensely from high-resolution data scaling, extreme long-term forecasting becomes increasingly dominated by systemic chaotic divergence, rendering pure data scaling progressively less effective. This formulation provides a rigorous guideline for balancing the computational cost of data synthesis against marginal performance gains across different forecasting horizons.

\begin{table}[htbp]
\centering
\caption{Fitted parameters of the empirical scaling law $E(D) = \alpha D^{-\beta} + E_{\infty}$ for representative atmospheric variables across different lead times. $R^2$ indicates the goodness of fit.}
\label{tab:scaling_law_params}
\resizebox{0.95\textwidth}{!}{
\begin{tabular}{llcccc}
\toprule
\textbf{Lead Time} & \textbf{Variable} & \textbf{Scaling Const ($\alpha$)} & \textbf{Exponent ($\beta$)} & \textbf{Irreducible Error ($E_{\infty}$)} & \textbf{$R^2$} \\
\midrule
6h & T2M & 0.817 & 2.207 & 0.682 & 0.968 \\
6h & U10 & 1.697 & 2.681 & 0.690 & 0.974 \\
6h & MSL & 41.950 & 1.192 & 33.370 & 0.959 \\
6h & Z500 & 35.062 & 1.012 & 20.684 & 0.951 \\
\midrule
72h & T2M & 1.164 & 1.189 & 1.107 & 0.938 \\
72h & U10 & 1.192 & 0.745 & 1.549 & 0.795 \\
72h & MSL & 268.332 & 1.126 & 144.224 & 0.802 \\
72h & Z500 & 236.864 & 0.879 & 122.599 & 0.785 \\
\midrule
120h & T2M & 1.146 & 0.792 & 1.488 & 0.870 \\
120h & U10 & 1.427 & 0.481 & 2.244 & 0.749 \\
120h & MSL & 270.943 & 0.767 & 277.567 & 0.689 \\
120h & Z500 & 285.245 & 0.542 & 239.073 & 0.693 \\
\bottomrule
\end{tabular}
}
\end{table}

\subsection{Extended Discussion on Physical Consistency of SR-Generated Fields}
The forecast improvements shown in Figs.~4--8 of the main paper demonstrate
that SR-generated pseudo-labels provide effective high-resolution training
signals for downstream forecasting. However, these downstream gains alone do
not fully rule out the concern that SR may hallucinate physically inconsistent
high-frequency details. We therefore directly examine whether the SR-generated
fields preserve key diagnostic relationships of the reference atmosphere, using
three complementary physical perspectives.

First, we evaluate the \textbf{moisture structure} using relative humidity (RH)
diagnosed from temperature and specific humidity, which tests whether
thermodynamic and moisture variables remain mutually coherent after SR.
Second, we evaluate the \textbf{dynamical structure} using geostrophic wind speed
$|\mathbf{V}_g|$ diagnosed from geopotential height $Z$, which reflects whether
the reconstructed pressure-level fields preserve large-scale balanced dynamics.
Third, we evaluate the \textbf{thermodynamic stratification} using potential
temperature differences $\Delta\theta$ between pressure levels, which measures
whether the vertical thermal structure remains consistent.

Table~\ref{tab:physical_consistency} reports latitude-weighted correlations
between the SR-generated diagnostics and those computed from the 0.1$^\circ$
analysis. The consistently high correlations indicate that the per-variable SR
fields preserve key moisture, dynamical, and thermodynamic structures despite
being generated independently for each variable. In addition, the power spectral
analysis in Fig.~\ref{fig:psd} shows that the SR fields closely match the
0.1$^\circ$ analysis over most resolved wavelengths, without introducing
systematic spurious high-frequency energy.

\noindent
\begin{minipage}{\linewidth}
\centering
\captionsetup{justification=centering}
\captionof{table}{Physical consistency of per-variable SR fields. Values are latitude-weighted correlations against EC 0.1$^\circ$ analysis.}
\label{tab:physical_consistency}
\scriptsize
\setlength{\tabcolsep}{3.0pt}
\renewcommand{\arraystretch}{0.92}
\begin{tabular}{c c | c c | c c}
\toprule
Level (hPa) & RH
& Level (hPa) & $|\mathbf{V}_g|$
& Layer (hPa) & $\Delta\theta$ \\
\midrule
500 & 0.962 & 500 & 0.821 & 700--500 & 0.972 \\
850 & 0.940 & 850 & 0.843 & 850--700 & 0.964 \\
925 & 0.961 & 925 & 0.855 & 925--850 & 0.937 \\
\bottomrule
\end{tabular}
\end{minipage}

\begin{figure}
    \centering
    \includegraphics[width=0.8\linewidth]{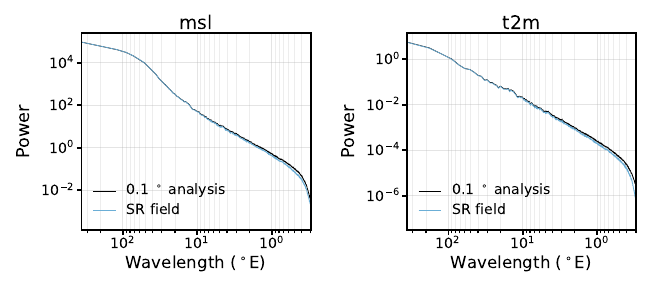}
    \caption{Power spectra of SR-generated fields and EC 0.1$^\circ$ analysis fields, using 2-meter temperature (t2m) and mean sea-level pressure (msl) as examples.}
    \label{fig:psd}
\end{figure}

\subsection{Extended Discussion on Baseline Selection and Forecasting Dynamics}
In Section 5.3 of the main text, our comparison includes IFS-HRES and Baguan+swin2sr over the full forecast range, and further extends to GraphCast+swin2sr, Pangu+swin2sr, and Baguan+GHR within the short-term 3-day window. It is essential to clarify that these ``+swin2sr'' configurations represent a \textbf{zero-shot post-processing paradigm}. In this setup, the respective $0.25^\circ$ pre-trained forecasting models are executed entirely frozen, and their output states are subsequently upscaled to $0.1^\circ$ during inference using our super-resolution model, without any end-to-end fine-tuning on the high-resolution grid. The Pangu and GraphCast forecasts are generated using
ECMWF's \texttt{ai-models} framework and its public model plugins\cite{Raoult_ai-models}.

This specific evaluation design is intentional, serving to expose the severe information bottleneck inherent to decoupled post-processing routes. At early forecasting lead times (short-range), the atmospheric state still contains rich high-frequency spatial structures and sharp mesoscale gradients. Coarse-resolution ($0.25^\circ$) models inherently fail to resolve these features, resulting in overly smoothed output states. Attempting to reconstruct such lost fine-scale information post hoc via super-resolution is fundamentally ill-posed and tends to amplify forecast errors. This information loss helps explain why our end-to-end \NAME model achieves its largest advantages in the short-to-medium range, where fine-scale structures are still most prominent.

Interestingly, as the auto-regressive rollout extends to longer lead times (e.g., 10 days), the performance gap between \NAME and the post-processing baselines gradually narrows. This trend is closely related to the chaotic nature of atmospheric dynamics and the accumulated error growth in deterministic forecasting. At extended lead times, all deterministic auto-regressive models, including \NAME, tend to lose variance and produce progressively smoother forecast states. As fine-scale structures become increasingly attenuated, the disadvantage imposed by the $0.25^\circ$ information bottleneck becomes less pronounced, which can make post-hoc upscaling appear relatively more competitive in the long range.

To ensure a fair comparison and to show that our gains are not simply due to avoiding post-processing smoothing, we also include \textbf{Baguan+GHR} (reproduced following \cite{han2024fengwu}) in our evaluation. This baseline represents a strong coarse-to-fine transfer learning paradigm, in which the forecasting model is further fine-tuned on the $0.1^\circ$ analysis grid. By consistently outperforming both the zero-shot post-processing baselines and the fully fine-tuned Baguan+GHR model within the critical 3-day operational window, our results indicate that the proposed ``synthetic-plus-real'' data scaling strategy provides a more effective route toward high-resolution AI weather forecasting.

\subsection{Comparison Against Station Observation}

While the grid-to-grid evaluations against EC analysis in the main text provide a comprehensive global perspective, gridded analysis data may still contain model-dependent diagnostic uncertainties. To further assess whether the $0.1^\circ$ resolution of \NAME translates into tangible improvements in real-world forecasting, we extend the evaluation to direct ground station measurements through grid-to-point verification.

For this evaluation, we construct a global ground-truth subset comprising measurements from \textbf{13,058} weather stations worldwide. The raw station observations were accessed via the \texttt{Meteostat} API\cite{meteostat}, which aggregates data from public providers including NOAA and DWD. We use hourly \texttt{wspd} (wind speed, corresponding to $\sqrt{u10^2 + v10^2}$) and hourly \texttt{temp} (air temperature, corresponding to T2M) from January 1, 2025 to December 20, 2025. To map gridded model outputs to station coordinates, we apply standard bilinear interpolation from the surrounding grid points before computing forecast errors.

Consistent with the grid-based evaluation, \NAME also performs strongly in station-based verification. As shown in Fig.~\ref{fig:station_eval}, \NAME achieves lower RMSE than both IFS-HRES and Baguan+swin2sr for 2-meter temperature and wind speed across most lead times. This result further supports the robustness of \NAME in localized surface forecasting.

\begin{figure}[htbp]
    \centering
    \includegraphics[width=\textwidth]{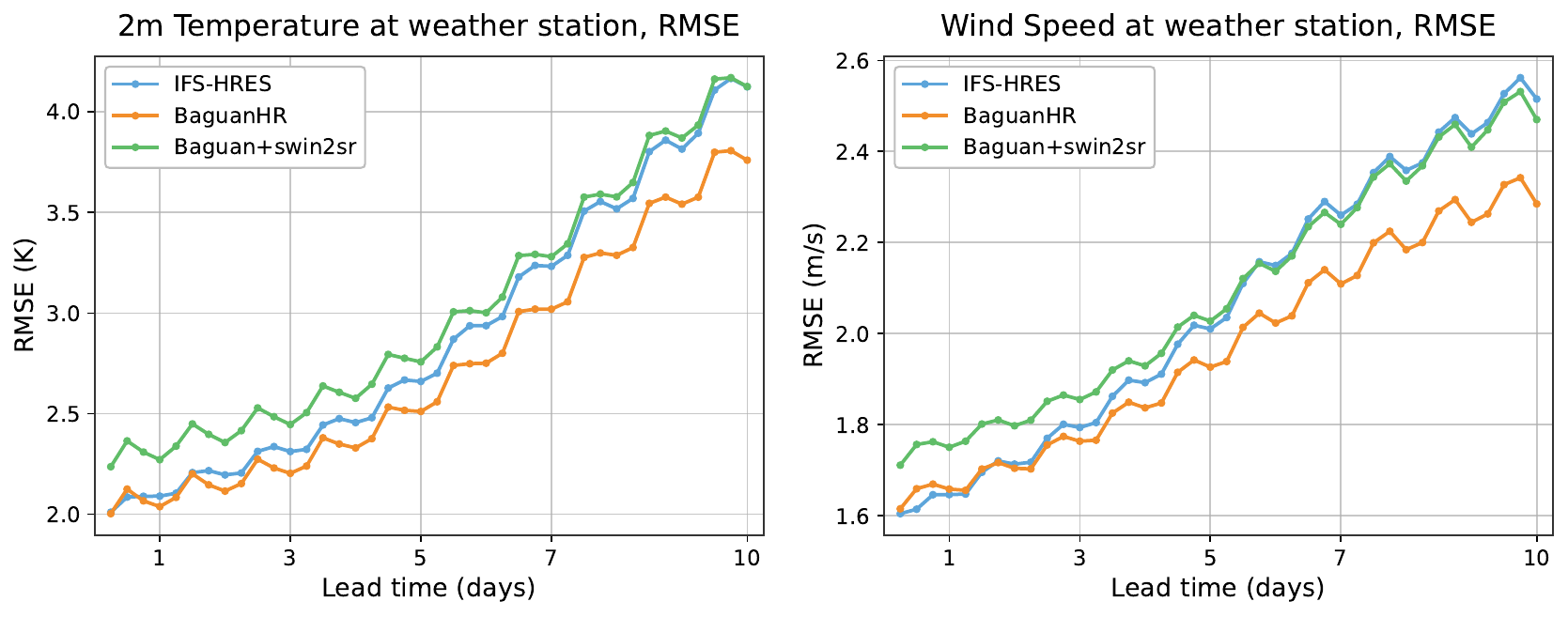}
    \caption{Station-based RMSE as a function of lead time for 2-meter temperature (left) and wind speed (right). \NAME consistently achieves lower errors than both IFS-HRES and Baguan+swin2sr across most lead times.}
    \label{fig:station_eval}
\end{figure}

\subsection{Visualization of High-Resolution Forecasts}

To qualitatively demonstrate the structural fidelity and precision of our proposed framework, Figures~\ref{fig:visualization_t2m}, \ref{fig:visualization_u10}, and \ref{fig:visualization_q_850} present comprehensive spatial visualizations of the 2-meter temperature (T2M), the 10-meter U-component of wind (U10), and specific humidity at 850 hPa (Q850), respectively. All forecast trajectories are initialized at \textbf{00:00 UTC on December 10, 2025}.

Each figure compares the ground truth (ECMWF $0.1^\circ$ analysis) against the predictions generated by three distinct systems: our end-to-end \NAME model, the zero-shot post-processing baseline (Baguan+swin2sr), and the operational physical baseline (IFS-HRES). Corresponding error distribution maps (Predicted $-$ Analysis) are also provided. These examples serve as qualitative case studies, showing that \NAME generally produces predictions that are visually closer to the analysis and exhibits lower overall error than the post-processing baseline.

\begin{figure}[htbp]
    \centering
    \includegraphics[width=\textwidth]{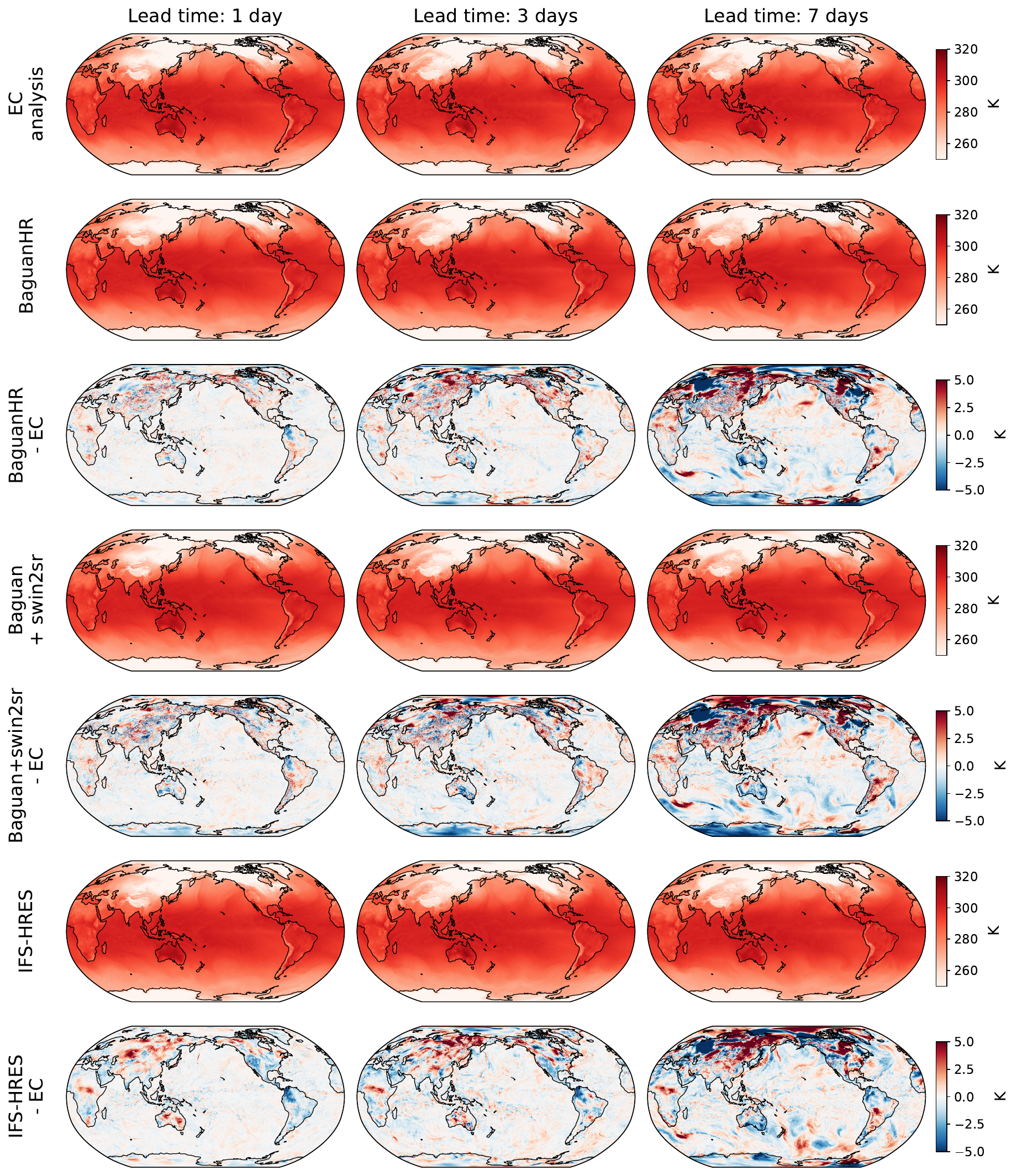}
    \caption{Qualitative comparison of \textbf{2-meter temperature (T2M)} forecasts initialized at 00:00 UTC, Dec 10, 2025.}
    \label{fig:visualization_t2m}
\end{figure}

\begin{figure}[htbp]
    \centering
    \includegraphics[width=\textwidth]{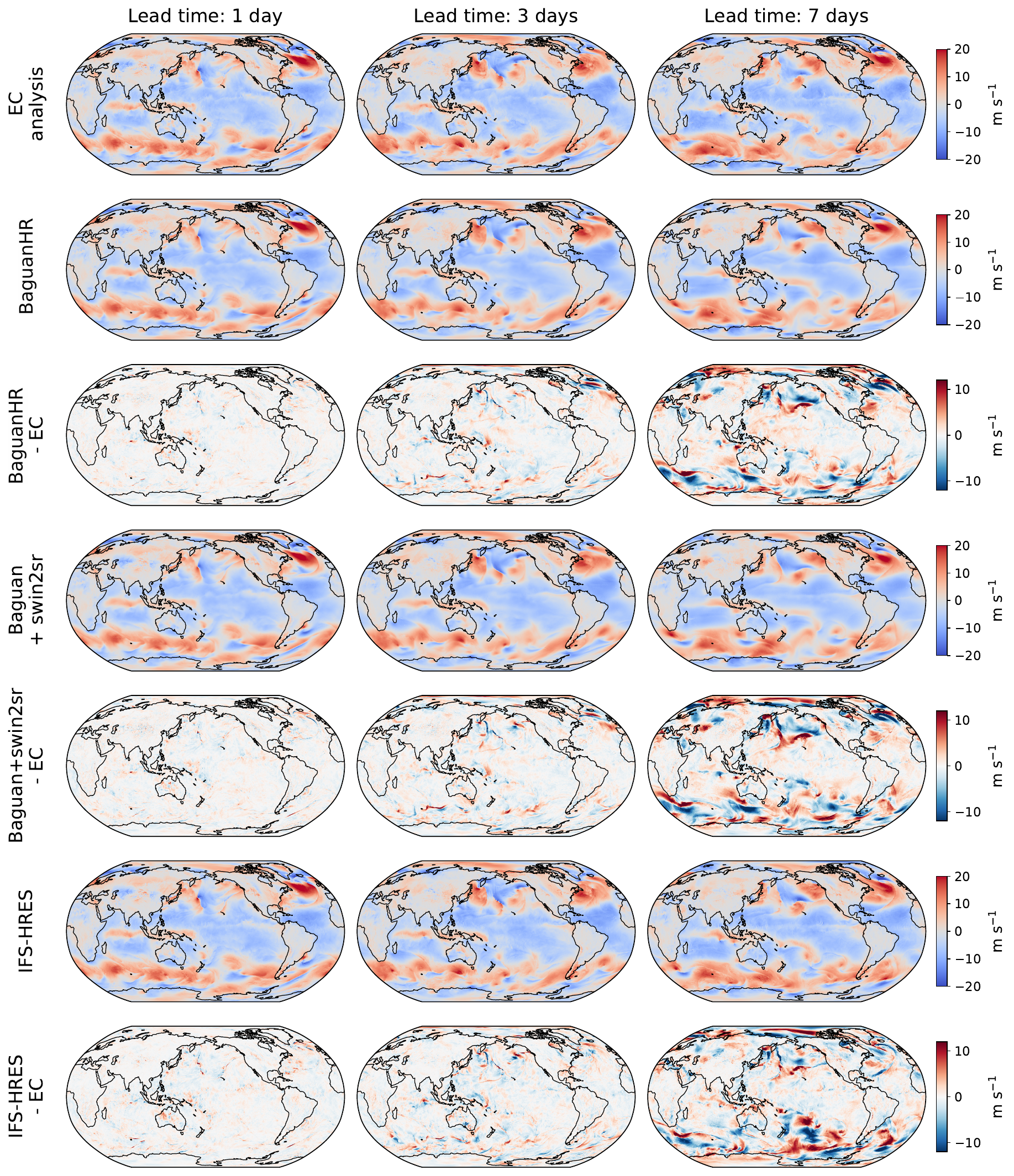}
    \caption{Qualitative comparison of the \textbf{10-meter U-component of wind (U10)} initialized at 00:00 UTC, Dec 10, 2025.}
    \label{fig:visualization_u10}
\end{figure}

\begin{figure}[htbp]
    \centering
    \includegraphics[width=\textwidth]{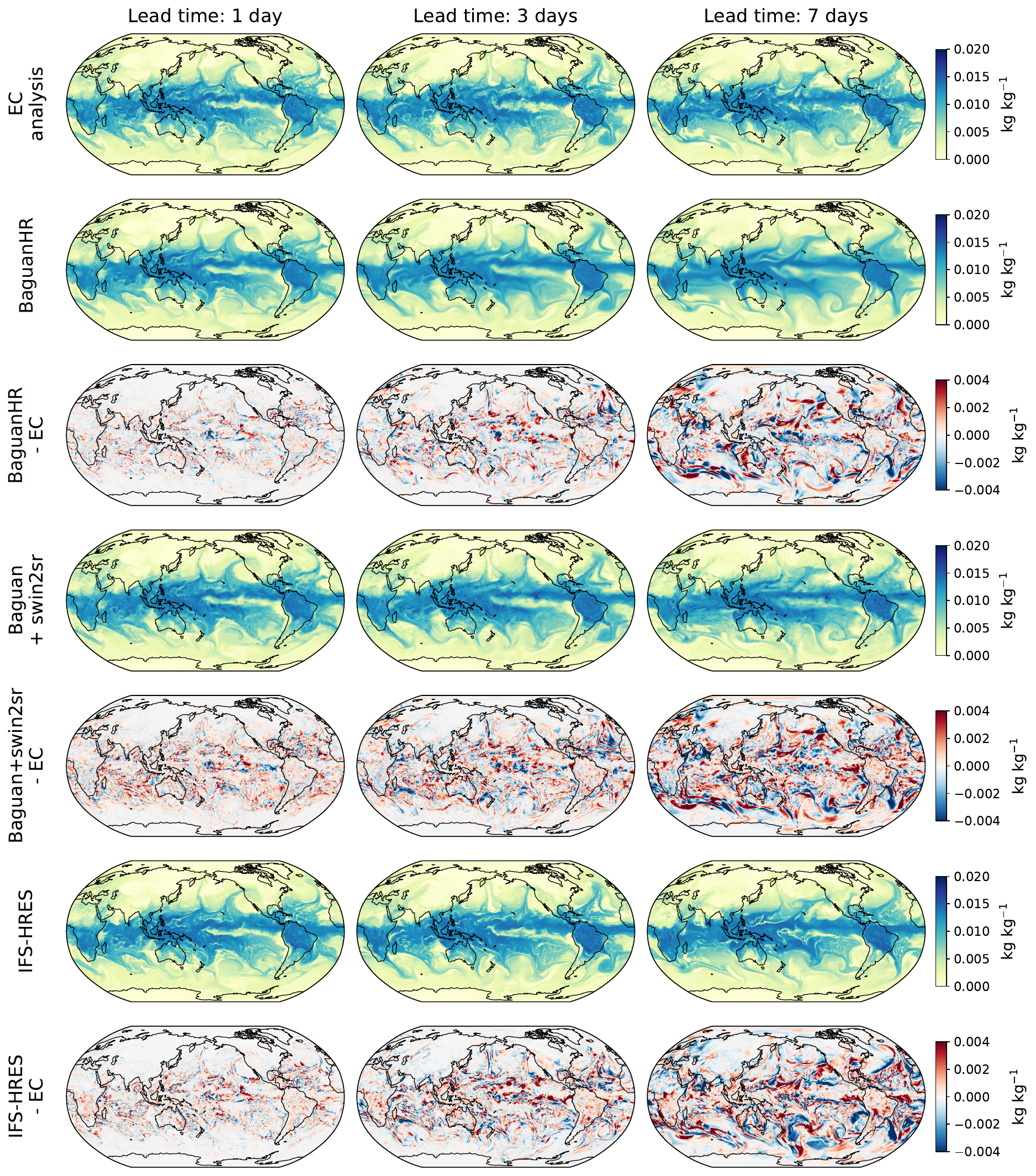}
    \caption{Qualitative comparison of \textbf{specific humidity at 850 hPa (Q850)} initialized at 00:00 UTC, Dec 10, 2025.}
    \label{fig:visualization_q_850}
\end{figure}

\subsection{The Performance of \NAME in Tracking Tropical Cyclones}

\begin{figure}[!t]
    \centering
    \includegraphics[width=\textwidth]{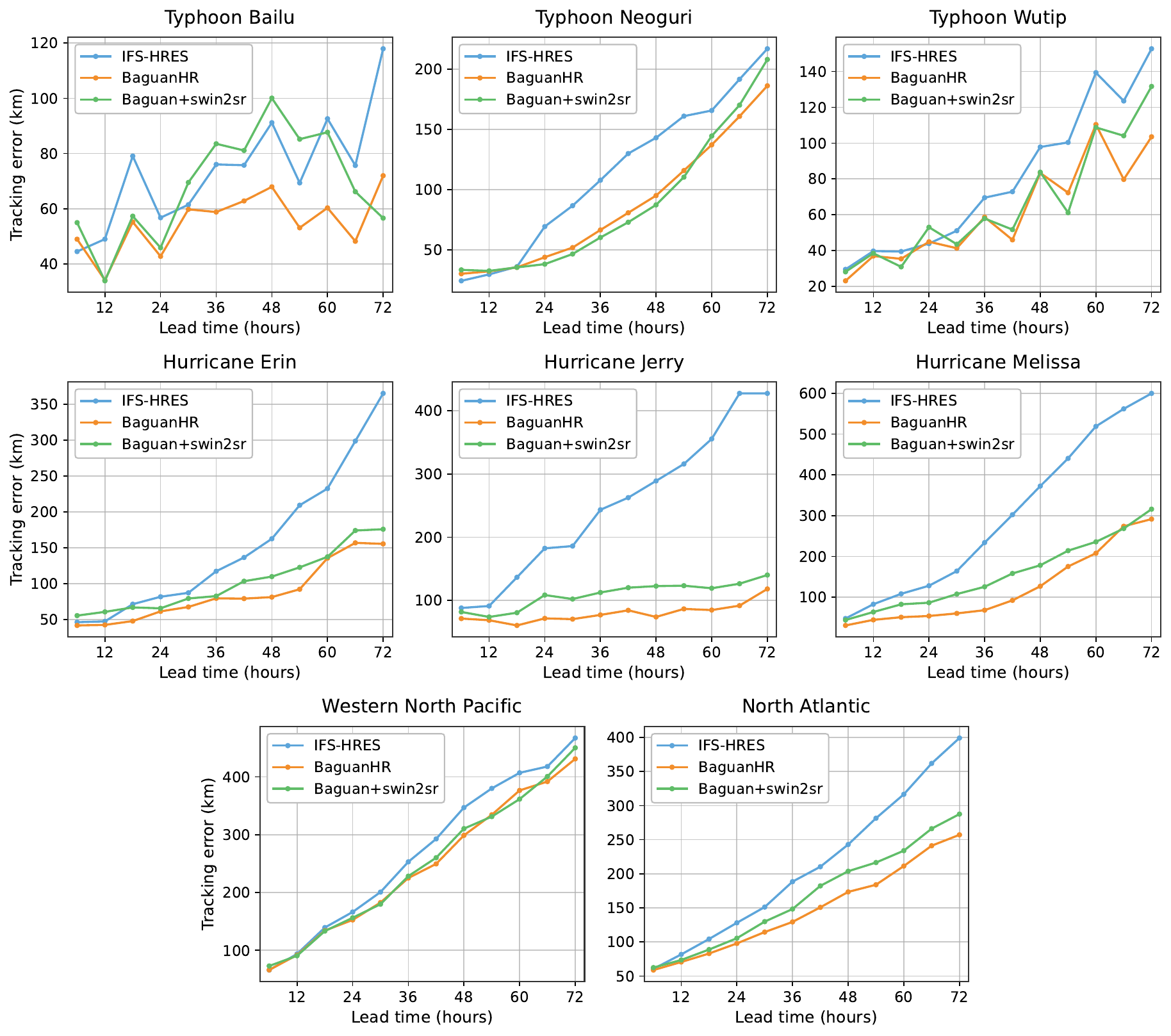}
    \caption{Comparison of tropical cyclone tracking error versus lead time for six representative storms and the corresponding basin-wide averages over the Western North Pacific and North Atlantic. The ground truth tracks are derived from the IBTrACS dataset. Errors for individual storms are averaged across multiple viable initialization timestamps.}
    \label{fig:case_study_typhoon}
\end{figure}

Building on the individual case study in Fig.~8(a) and (b) of the main text, we further evaluate the robustness of our framework for tropical cyclone (TC) tracking across a broader set of historical storms. Figure~\ref{fig:case_study_typhoon} shows tracking error (km) as a function of forecast lead time for six representative storms, together with basin-mean errors over the two major cyclogenesis regions: the Western North Pacific and the North Atlantic. The ground-truth TC tracks used in both the main text and this supplementary analysis are taken from the International Best Track Archive for Climate Stewardship (IBTrACS) \cite{knapp2010international, gahtan2024ibtracs}, which provides the most comprehensive and globally unified archive of historical tropical cyclone best-track data.

For a more statistically robust evaluation, the tracking error of each storm is not computed from a single forecast only. Instead, we generate multiple forecast trajectories initialized from different valid 00:00/12:00 UTC timestamps over the storm lifetime, and report the mean tracking error across these initialization times at each lead time. As shown in Fig.~\ref{fig:case_study_typhoon}, \NAME achieves consistently strong TC tracking performance across both individual storms and basin-mean statistics. In particular, over the North Atlantic and for several representative hurricanes, \NAME generally maintains lower tracking errors than both Baguan+swin2sr and IFS-HRES throughout the 72-hour window. Although the margin is smaller over the more challenging Western North Pacific, \NAME remains competitive overall, indicating that synthetic-data-based high-resolution scaling provides a reliable basis for TC trajectory prediction.

%% file: sections/app6_limitations.tex
This study relies entirely on public, gridded meteorological datasets and does not involve human subjects, personal data, or sensitive individual information. More broadly, our framework provides a practical data-centric route for improving high-resolution weather forecasting under limited compute and data availability. By converting long-term coarse-resolution reanalyses into synthetic 0.1° training data, it may help widen access to stronger extreme-weather forecasting systems for institutions and regions that cannot afford large operational modeling pipelines.

Nevertheless, several limitations remain. First, our synthetic high-resolution
data are generated by per-variable super-resolution models trained independently
for each field. While this design improves variable-specific fidelity, it may
not fully preserve cross-variable coupling and multivariate physical consistency
at the finest scales. Second, although 0.1$^\circ$ resolution (roughly 9 km) is
effective for mesoscale forecasting, it is still insufficient to explicitly
resolve microscale convective extremes such as tornadoes, hail, or highly
localized severe storms. Third, forecast errors still accumulate with lead time
under auto-regressive rollout. In particular, \NAME brings limited additional
gains at 10--15 day lead times, consistent with the practical 10-day
predictability limit of instantaneous midlatitude weather~\cite{zhang2019predictability};
at such horizons, forecasts are mainly controlled by large-scale atmospheric
patterns and fine-scale details become weakly predictable. Future work will
explore ensemble forecasting to better characterize forecast uncertainty and
extend the model's practical predictability.